\PassOptionsToPackage{table}{xcolor}
\documentclass[journal, twocolumn]{IEEEtran}

\usepackage{amsmath,amsfonts}
\usepackage{array}
\usepackage{textcomp}
\usepackage{stfloats}
\usepackage{url}
\usepackage{verbatim}
\usepackage{graphicx}
\usepackage{subfiles}
\usepackage{subcaption}
\usepackage{booktabs}
\usepackage{multirow}
\usepackage{makecell}
\usepackage{algorithm}
\usepackage{algpseudocode}
\usepackage{enumitem}
\usepackage{fmtcount}
\usepackage{tcolorbox}
\usepackage{cite}
\usepackage{amssymb}
\usepackage{xcolor}
\usepackage{hyperref}
\usepackage{tabularx}
\usepackage{amsthm}

\usepackage{pifont}

\newtheorem{theorem}{Theorem}[section]
\newtheorem{proposition}[theorem]{Proposition}
\newtheorem{definition}[theorem]{Definition}

\newtheorem{problem}{Problem}

\newtheorem*{remark}{Remark}

\usepackage{tikz}
\usetikzlibrary{shapes.geometric, arrows.meta, positioning, fit, backgrounds, calc, decorations.pathreplacing, shapes.multipart}

\let\origincludegraphics\includegraphics
\renewcommand{\includegraphics}[2][]{%
  \IfFileExists{#2}{%
    \origincludegraphics[#1]{#2}%
  }{%
    \fbox{\parbox{0.8\linewidth}{\centering\small\texttt{[Missing figure: #2]}}}%
  }%
}

\def\BibTeX{{\rm B\kern-.05em{\sc i\kern-.025em b}\kern-.08em
    T\kern-.1667em\lower.7ex\hbox{E}\kern-.125emX}}
\usepackage{balance}

\begin{document}

\title{\textsc{RECTOR}: Priority-Aware Rule-Based Reranking
for Compliance-Aware Autonomous Driving Trajectory Selection}

\author{Hadi Hajieghrary$^{1}$, Benedikt Walter$^{2}$, Chaitanya Shinde$^{3}$, Paul Schmitt$^{4}$, and Miguel Hurtado$^{5}$% <-this % stops a space
\thanks{$^{1}$Hadi Hajieghrary {\tt\small Hadi.Hajieghrary@Torc.ai}, $^{2}$Benedikt Walter {\tt\small Benedikt.Walter@Torc.ai}, $^{3}$Chaitanya Shinde {\tt\small Chaitanya.Shinde@Torc.ai}, and $^{5}$Miguel Hurtado {\tt\small Miguel.Hurtado@Torc.ai} are with TORC Robotics LLC, an independent subsidiary of Daimler Truck AG. $^{4}$Paul Schmitt {\tt\small pauls@massrobotics.org} is with  and Reynolds \& Moore and MassRobotics. 

This paper represents research activities conducted within TORC Robotics LLC. It does not describe any production system, does not constitute the safety case for any Daimler Truck AG product or development program, and has not been reviewed by Daimler Truck AG for regulatory compliance or product liability purposes. The research was conducted using public benchmark data and does not reflect the proprietary safety architecture of any deployed or in-development TORC Robotics or Daimler Truck product.}%
 }

\maketitle

\begin{abstract}
Autonomous driving stacks must pick one trajectory from a multi-modal candidate set; choosing by model confidence ignores safety, traffic-law, and comfort constraints. We present \textsc{RECTOR} (Rule-Enforced Constrained Trajectory Orchestrator), a post-generation reranking layer that scores candidates against a tiered rulebook (Safety~$\succ$~Legal~$\succ$~Road~$\succ$~Comfort) via differentiable proxies and a scene-conditioned applicability mechanism, then selects with a deterministic $\varepsilon$-lexicographic rule that preserves cross-tier priority by construction---without retraining the predictor.

On the Waymo Open Motion Dataset \texttt{validation\_interactive} split (43{,}219 augmented instances, $K{=}6$), under Protocol~B (28-rule proxy catalog, oracle applicability) rule-aware selection cuts Safety+Legal violations from 28.58\% to 20.42\% and Total from 40.32\% to 32.41\% versus confidence-only on the same candidates. A uniform-weight weighted-sum baseline matches binary compliance on this benchmark---the empirical lift comes from rule-aware ranking, while the lexicographic guarantee is the structural differentiator no weight calibration can replicate. Under adversarial confidence corruption, confidence-only selection fails in 100\% of scenarios while both rule-aware selectors reject the injected mode in $\sim$96\%. All figures are proxy-evaluator results (not a safety certificate), open-loop, 5\,s horizon, U.S.\ rules, validation split. Source code: \cite{rector_code}.
\end{abstract}


\section{Introduction}
\label{Sec:Introduction}
\IEEEPARstart{A}utonomous driving stacks must ultimately commit to a single trajectory from a multimodal candidate set. Modern predictors generate several plausible futures, but the final choice is still typically made by confidence or likelihood. That optimizes average geometric accuracy, not behavioral acceptability: a high-confidence mode can still run a red light, violate right-of-way, or tailgate. On WOMD \texttt{validation\_interactive}, best-of-$K$ error is also heavy-tailed: minADE is 0.46\,m at the median but 2.16\,m and 3.62\,m at the 95th and 99th percentiles. In these long-tail interactions, confidence is a weak ranking signal, and explicit constraints become decisive.

This paper studies \emph{open-loop} selection over a fixed candidate set, after generation. We present \textsc{RECTOR} (Rule-Enforced Constrained Trajectory ORchestrator), a post-generation reranking layer that orders candidates under the priority Safety~$\succ$~Legal~$\succ$~Road~$\succ$~Comfort using an $\varepsilon$-lexicographic rule. RECTOR has three pieces: a scene-conditioned applicability mechanism over 28 rules (retained as a tested reference component, not the deployment-default; see Section~\ref{Sec:Accuracy_SectionRule-Compliance_Behavior_and_Efficiency}), 24 differentiable rule proxies for batched evaluation, and a deterministic tolerance-based lexicographic selector. The selector is interchangeable with any predictor that outputs world-frame candidates and confidences.

The central claim is that proxy-evaluator compliance can be improved \emph{at selection time} when the candidate set contains an acceptable option. Generation and selection are distinct responsibilities: the generator covers plausible futures; the reranker picks the least-violating trajectory under explicit priorities, flags infeasibility when none is fully compliant, and avoids collapsing heterogeneous rules into a single penalty. This requires no online constrained optimization or changes to learned scene features. RECTOR does not enforce rules in the control-theoretic or formal-verification sense: the guarantee it provides is \emph{structural priority preservation over normalized proxy tier scores}, conditional on the applicability mask and the reliability of each proxy (see Section~\ref{Sec:Section_VIII_Verification_Invariants_Numerical_Stability_and_Integration_Tests} for proxy--auditor concordance and known false-negative rates, notably on L0.R3). All compliance numbers in this paper are open-loop, 5\,s horizon, on WOMD \texttt{validation\_interactive}; closed-loop reactive evaluation is outside the scope of this study.

\subsection{Related Work and Positioning}
\label{subsec:related_work}

Most prior work improves the \emph{candidate set} rather than the executed mode. Anchor-based predictors (LaneGCN, TNT/DenseTNT), Transformer set predictors (MTR, Wayformer, QCNet, GameFormer), and latent-variable models (Trajectron++, MotionLM) target multimodal coverage~\cite{liang2020lanegcn,zhao2021tnt,gu2021densetnt,salzmann2020trajectron,shi2022motion,nayakanti2023wayformer,zhou2023qcnet,huang2023gameformer,seff2023motionlm}. Constraint-guided generators (CTG, CTG++, MotionDiffuser) bias generation toward compliant regions~\cite{zhong2023ctg,zhong2023ctgpp,jiang2023motiondiffuser}. These methods strengthen the proposal distribution but do not impose a priority hierarchy on the single trajectory that is executed.

Constraints are also enforced earlier in the stack. Planning- and control-time methods such as MPC, control barrier functions, reachability analysis, and rulebook-based control embed constraints in online optimization~\cite{borrelli2005mpc,ames2017control,althoff2010reachability}. Policy-time methods such as constrained RL and SafetyNet internalize safety during learning~\cite{altman1999constrained,achiam2017constrained,vitelli2022safetynet}. Both target different intervention points: RECTOR addresses post-generation reranking while leaving the predictor unchanged.

The closest comparison is scalar reranking. Weighted penalties and differentiable optimization layers (OptNet, differentiable MPC) rank pre-generated candidates, but cross-tier behavior depends on weight calibration~\cite{marler2004survey,amos2017optnet,amos2018differentiable}. RSS and liability rulebooks motivate explicit priority structure~\cite{shalev2017formal,censi2019liability} but do not address learned scene-conditioned applicability over motion candidates. RECTOR targets post-generation selection on a fixed candidate set, predicts applicability from scene context, scores with smooth rule proxies, and enforces priority structurally rather than through tuned weights. Consistent with that focus, we evaluate selected-trajectory accuracy and rule compliance, not only candidate-set coverage.

\subsection{Contributions}
\label{subsec:contributions}
%------------------------------------------------------------

In the reference instantiation, a Transformer--CVAE generator inspired by M2I produces $K{=}6$ candidates over a 5\,s horizon ($T{=}50$ at 10\,Hz), evaluated against $R{=}28$ rules in four tiers (Safety: 5, Legal: 7, Road: 2, Comfort: 14); 24 are scored by differentiable proxies and 4 are audit-only. Candidates are ranked lexicographically, and confidence breaks ties that are compliance-equivalent; the selector is deterministic, parameter-free, and leaves the generator unchanged. The contributions are:

\begin{enumerate}[label=\textbf{C\arabic*},leftmargin=2.5em,itemsep=2pt]

\item \textbf{Rule-aware selection measurably improves proxy compliance on a fixed candidate set.} Under Protocol~B (28 rules, oracle applicability), rule-aware reranking reduces Safety+Legal violations from 28.58\% to 20.42\% (8.16~pp) and Total from 40.32\% to 32.41\% (7.91~pp) versus confidence-only on 43{,}219 instances. The deployed Protocol~A (24 rules, learned head) shows larger apparent gains, but its inflated baseline reflects Safety over-activation and Legal/Road suppression by the head; we therefore lead with Protocol~B. A uniform-weight weighted-sum baseline matches the lexicographic binary compliance on this benchmark, so the empirical lift is rule-aware over confidence-only.

\item \textbf{$\varepsilon$-Lexicographic selection with a weight-free cross-tier guarantee.} The selector preserves Safety~$\succ$~Legal~$\succ$~Road~$\succ$~Comfort over the normalized tier scores within tolerance $\varepsilon$; no scalar reweighting can replicate this. The guarantee is structural and proxy-conditional, not a control-theoretic or jurisdictional safety certificate. Lexicographic and weighted-sum match here only because cross-tier conflicts are rare in WOMD (\emph{does not}, not \emph{cannot}; Sec.~\ref{Sec:Accuracy_SectionRule-Compliance_Behavior_and_Efficiency}). The selector also produces a regression-testable decision trace.

\item \textbf{Structural robustness to upstream confidence miscalibration.} Under an adversarial protocol that injects a Safety-violating candidate with the maximum confidence, confidence-only selection fails in 100\% of scenarios; both rule-aware selectors reject the injection in $\sim$96\% across three injection families---structural robustness no scalar recalibration can match.

\item \textbf{Differentiable proxy framework with explicit reliability bounds.} Pairwise concordance with the full Waymo evaluator is \textbf{0.875} on 10 variance-bearing rules (0.948 over all 24). Proxies run in $\sim$1\,ms. The collision-avoidance proxy L0.R3 attains 0.80 concordance with 80.7\% FNR at 0.0\% FPR; compliance figures involving L0.R3 are proxy-evaluator rates, and a post-selection full-evaluator audit is recommended for any deployment-facing reading (Sec .~\ref {Sec:Section_VIII_Verification_Invariants_Numerical_Stability_and_Integration_Tests}).

\item \textbf{Always-on applicability as a design recommendation (negative result on learning applicability).} A parameter-free always-on policy for Safety/Legal Pareto-dominates the learned 3.33\, M-parameter head on cross-evaluated compliance (S+L 20.53\% vs.\ 20.69\%) \emph{and} on accuracy (selADE 2.481\,m vs.\ 2.737\,m); we therefore recommend always-on Safety/Legal as the default and present the learned head as an ablation.

\item \textbf{Verification and practical integration.} Formal properties (order invariance, infeasibility fallback, scalarization correctness; Propositions~II.2--II.3), 99.2--100\% invariant pass rates, 0.0\% post-mitigation NaN/Inf, and at most 6.1~pp / 12.1~pp violation increase under perception/map perturbations. Mean inference latency is 7.3\,ms per sample (137~samples/s); rule evaluation and selection add only 1.4\,ms (19\%). Rule- and tier-level decision traces support audit, and a world-frame output contract keeps the RECTOR plug-in compatible.

\end{enumerate}

\begin{figure}[t]
\centering
\includegraphics[width=1.0\columnwidth]{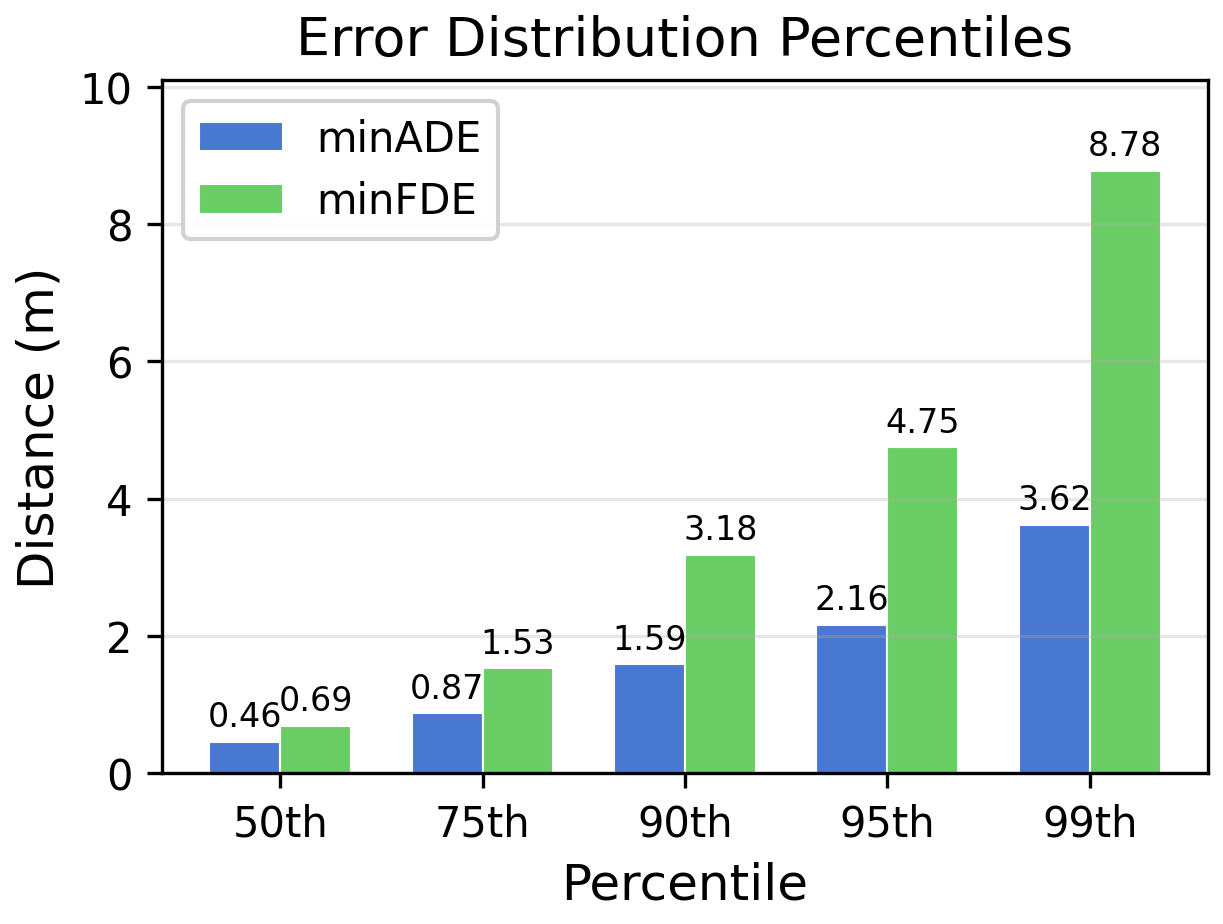}
\caption{Percentiles of best-of-$K$ geometric error on WOMD \texttt{validation\_interactive} with $K{=}6$. Blue bars show minADE and green bars show minFDE. The plotted minADE percentiles are 0.46\,m (50th), 0.87\,m (75th), 1.53\,m (90th), 2.16\,m (95th), and 3.62\,m (99th), illustrating the heavy tail that motivates rule-aware selection.}
\label{fig:percentile_analysis}
\end{figure}

\textsc{RECTOR} reranks one trajectory from $K{=}6$ candidates over a 5\,s open-loop horizon. Compliance gains come solely from reranking a fixed candidate set, and the open-loop, 5\,s, validation-split setting precludes any closed-loop or deployment-grade claim. Model selection is also performed on the validation set, so absolute generation and Protocol~A numbers carry an unquantified optimistic bias; Protocol~B selector comparisons use oracle applicability and are immune. Held-out test reporting under both protocols is a deliberate next step (Sec.~\ref{Sec:Conclusion}). Results use U.S.\ right-hand traffic conventions; other jurisdictions require recalibrating the rule catalog. Route-level and longer-horizon planning are out of scope. Figure~\ref{fig:rector_architecture} gives a system overview; Sec.~\ref{Sec:RECTOR} details the architecture.

\begin{figure*}[t]
\centering
\resizebox{\textwidth}{!}{%
\begin{tikzpicture}[
% ── Node styles ──────────────────────────────────────────────
input/.style={
rectangle, draw=gray!70, fill=gray!10,
minimum width=2.0cm, minimum height=0.7cm,
font=\footnotesize, align=center},
encoder/.style={
rectangle, rounded corners=3pt,
draw=blue!70!black, fill=blue!15,
minimum width=2.2cm, minimum height=1.8cm,
font=\small\bfseries, align=center},
decoder/.style={
rectangle, rounded corners=3pt,
draw=green!60!black, fill=green!15,
minimum width=2.2cm, minimum height=1.8cm,
font=\small\bfseries, align=center},
rulehead/.style={
rectangle, rounded corners=3pt,
draw=orange!80!black, fill=orange!15,
minimum width=2.0cm, minimum height=1.3cm,
font=\small\bfseries, align=center},
proxy/.style={
rectangle, rounded corners=3pt,
draw=red!70!black, fill=red!10,
minimum width=2.0cm, minimum height=1.3cm,
font=\small\bfseries, align=center},
scorer/.style={
rectangle, rounded corners=3pt,
draw=purple!70!black, fill=purple!15,
minimum width=2.0cm, minimum height=1.3cm,
font=\small\bfseries, align=center},
output/.style={
rectangle, draw=gray!70, fill=yellow!10,
minimum width=2.0cm, minimum height=0.7cm,
font=\footnotesize, align=center},
data/.style={
rectangle, draw=gray!50, fill=white,
minimum width=1.6cm, minimum height=0.55cm,
font=\scriptsize, align=center},
arrow/.style={->, >=Stealth, thick, gray!70},
dataarrow/.style={->, >=Stealth, semithick, gray!50},
lbl/.style={font=\scriptsize\itshape, text=gray!70}
]

% ═══════════════════════════════════════════════════════════
% ROW 1 — Generation pipeline
% ═══════════════════════════════════════════════════════════
\node[input] (ego)    at (0,  0) {Ego History\\$[\text{B}, 11, 4]$};
\node[input] (agents) at (0, -1) {Agent States\\$[\text{B}, 32, 11, 4]$};
\node[input] (lanes)  at (0, -2) {Lane Centers\\$[\text{B}, 64, 20, 2]$};

\begin{scope}[on background layer]
\node[draw=gray!40, dashed, rounded corners=5pt,
fit=(ego)(agents)(lanes), inner sep=5pt,
label={[font=\footnotesize\bfseries, text=gray!60]above:Inputs}] {};
\end{scope}

\node[encoder, right=0.7cm of agents] (m2i)
{M2I-style\\Scene Encoder\\[2pt]{\footnotesize (PointNet + Transformer)}};

\node[lbl, below=0.1cm of m2i] {\textbf{jointly trained}};

\node[data, below=0.7cm of m2i] (scene_emb)
{Scene Emb.\\$[\text{B}, 256]$};

\node[decoder, right=0.7cm of scene_emb] (cvae)
{CVAE\\Decoder};

\node[output, right=1.85cm of cvae, yshift=+0.4cm] (conf)
{Confidences\\$[\text{B}, 6]$};

\node[output, right=1.85cm of cvae, yshift=-0.4cm] (traj)
{Ego Trajectories\\$[\text{B}, 6, 50, 4]$};

\draw[arrow] (ego.east)    --++(0.25,0) |- ([yshift=+0.5cm]m2i.west);
\draw[arrow] (agents.east) -- (m2i.west);
\draw[arrow] (lanes.east)  --++(0.25,0) |- ([yshift=-0.5cm]m2i.west);
\draw[arrow] (m2i.south)      -- (scene_emb.north);
\draw[arrow] (scene_emb.east) -- (cvae.west);
\draw[arrow] (cvae.east) --++(0.25,0) |- (traj.west);
\draw[arrow] (cvae.east) --++(0.25,0) |- (conf.west);

% ═══════════════════════════════════════════════════════════
% ROW 2 — Rule-evaluation pipeline
% ═══════════════════════════════════════════════════════════
\node[rulehead, below=1.2cm of scene_emb] (apphead)
{Rule\\Applicability\\Head};

\node[data, right=0.5cm of apphead] (app_out)
{Applicability\\$\mathbf{a}\in[0,1]^{28}$};

\node[rectangle, draw=orange!60!black, fill=orange!5,
minimum width=1.5cm, minimum height=0.55cm,
font=\scriptsize, align=center,
right=0.5cm of app_out] (thresh)
{Threshold\\$\hat{\mathbf{a}}\in\{0,1\}^{28}$};

\node[proxy, right=0.5cm of thresh] (proxies)
{Smooth\\Rule Proxies};

\node[data, right=0.8cm of proxies] (viol)
{Violations\\$[\text{B}, 6, 24]^{*}$};

\node[scorer, right=0.5cm of viol] (scorer)
{Tiered\\Lexicographic\\Scorer};

\node[output, right=0.6cm of scorer, yshift=+0.4cm] (selected)
{Selected Ego\\Trajectory};

\node[output, right=0.6cm of scorer, yshift=-0.4cm] (explain)
{Rule-based\\Explanation};

\draw[arrow] (scene_emb.south) -- (apphead.north);
\draw[arrow] (apphead.east)    -- (app_out.west);
\draw[arrow] (app_out.east)    -- (thresh.west);
\draw[arrow] (thresh.east)     -- (proxies.west);
\draw[arrow] (proxies.east)    -- (viol.west);
\draw[arrow] (viol.east)       -- (scorer.west);
\draw[arrow] (scorer.east) --++(0.25,0) |- (selected.west);
\draw[arrow] (scorer.east) --++(0.25,0) |- (explain.west);

% ── Cross-row connections ────────────────────────────────
\draw[arrow] (traj.south) --++(0,-0.55)
-| ([xshift=-0.2cm]proxies.north);

\draw[dataarrow, dashed]
(lanes.south) --++(0,-1.7)
-| ([xshift=-0.6cm]proxies.north);

\node[lbl] at ([xshift=1.4cm, yshift=-1.4cm]lanes.south)
{\textbf{map context}};

\draw[arrow] (conf.east) -| (scorer.north);

\node[font=\scriptsize, text=purple!70!black, align=left,
anchor=north west]
at ([xshift=-0.75cm, yshift=-0.25cm]scorer.south west)
{Safety $\succ$ Legal $\succ$ Road $\succ$ Comfort};

\end{tikzpicture}%
}
\caption{\textbf{\textsc{RECTOR} system overview.}
Scene inputs flow through an M2I-style encoder to produce the scene embedding
$\mathbf{e}_{\text{scene}}\in\mathbb{R}^{256}$.
The CVAE decoder generates six candidate ego trajectories, each with an associated confidence score.
The rule applicability head $(3.33\,\mathrm{M}$ of $8.83\,\mathrm{M}$ total parameters)
predicts continuous applicability scores $\mathbf{a}\in[0,1]^{28}$,
thresholded to binary masks $\hat{\mathbf{a}}\in\{0,1\}^{28}$.
Smooth proxies score each candidate on $24$ of $28$ rules, with $4$ remaining audit-only; the tiered lexicographic scorer then selects the trajectory with the smallest tier-score vector
respecting Safety~$\succ$~Legal~$\succ$~Road~$\succ$~Comfort, with confidence entering only as
a tiebreaker among compliance-equivalent candidates.}
\label{fig:rector_architecture}
\end{figure*}
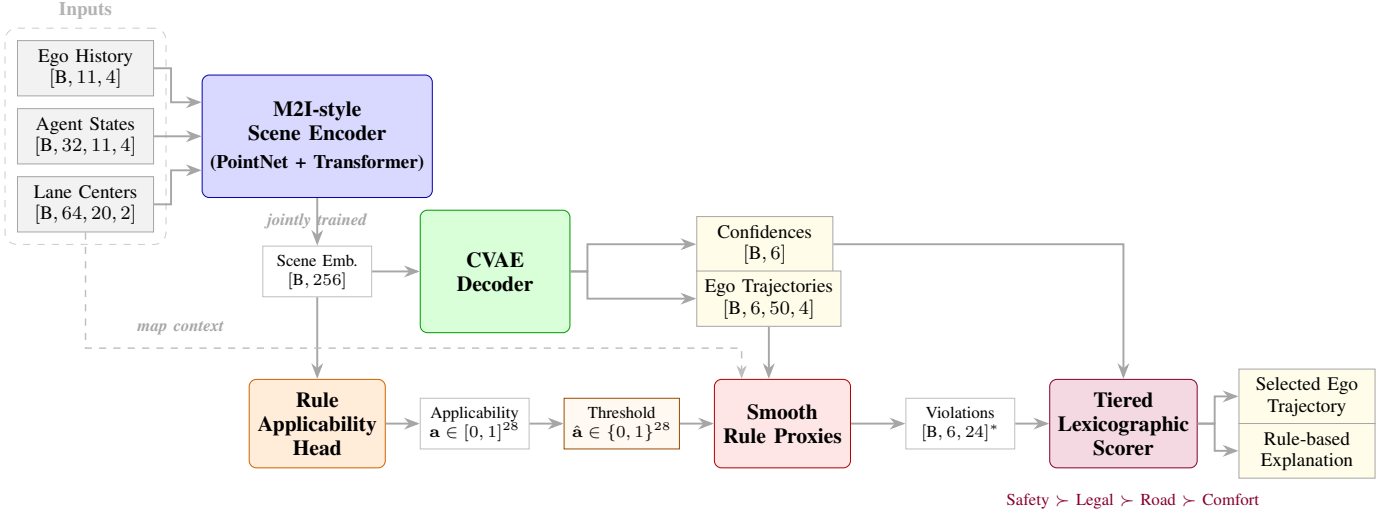

\section{Problem Formulation}
\label{Sec:Problem_Formulation}
\begin{table}
\centering
\small
\caption{Key Notation and Symbol Definitions}
\label{tab:notation}
\begin{tabularx}{\columnwidth}{@{}l>{\raggedright\arraybackslash}X@{}}
\toprule
\multicolumn{2}{l}{\textit{Scenarios and Trajectories}} \\
$s$ & Driving scenario sampled from $\mathcal{D}$ \\
$\mathbf{x}_{\text{ego}},\;\mathbf{X}_{\text{agents}},\;\mathcal{M}$ & Ego history, neighbor histories, map context \\
$\tau,\;\tau_t$ & Ego future trajectory; state at timestep $t$ \\
$\mathcal{T} = \{\tau^{(1)},\ldots,\tau^{(K)}\}$ & Candidate set ($K{=}6$) \\
$\tau^{\mathrm{gt}}$ & Ground truth future trajectory \\
$p^{(k)}$ & Confidence score for candidate $k$ \\
$T_h,\;T$ & History (11 steps, 1.1\,s); horizon (50 steps, 5.0\,s) \\
\midrule
\multicolumn{2}{l}{\textit{Rules and Compliance}} \\
$\mathcal{R},\;R{=}28$ & Rule catalog and its cardinality \\
$\ell(r) \in \{0,1,2,3\}$ & Tier assignment for rule $r$ \\
$a_r(s),\;\hat{a}_r(s)$ & Oracle / predicted applicability indicator \\
$V_r(\tau;s) \ge 0$ & Raw violation severity \\
$\bar{V}_r(\tau;s) \in [0,1]$ & Normalized severity: $c_r(V_r) \in [0,1]$; see \eqref{eq:normalization_contract} \\
\midrule
\multicolumn{2}{l}{\textit{Tier Scores and Selection}} \\
$\tilde{w}_r$ & Intra-tier weight; $\sum_{r:\ell(r)=\ell}\tilde{w}_r = 1$ \\
$S_\ell(\tau;s) \in [0,1]$ & Aggregated tier-$\ell$ score \\
$\mathbf{S}(\tau;s)$ & Tier score vector $[S_0, S_1, S_2, S_3]$ \\
$\varepsilon_\ell$ & Tier tolerance for near-tie resolution \\
$L{=}4$ & Number of tiers \\
\midrule
\multicolumn{2}{l}{\textit{Metrics}} \\
$\text{minADE},\;\text{minFDE}$ & Best-of-$K$ displacement errors \\
$\text{selADE},\;\text{selFDE}$ & Selected-trajectory displacement errors \\
$\text{MissRate}_\delta$ & Fraction with $\text{minFDE} > \delta$ \\
\bottomrule
\end{tabularx}
\end{table}

We formulate trajectory selection as a hierarchical decision problem over a fixed set of candidates. The generator provides diverse futures; the selector selects the trajectory most acceptable according to a prioritized rulebook. We fix the notation, motivate the rule-based objective by showing that geometric metrics are insufficient, and formalize the lexicographic criterion along with its correctness properties. Table~\ref{tab:notation} summarizes notation used throughout the paper.
 
%-------------------------------------------------------------------------------
\subsection{Trajectory Prediction Setting}
\label{subsec:prediction_setting}
%-------------------------------------------------------------------------------
 
Let $s \sim \mathcal{D}$ denote a driving scenario. A multi-modal predictor $f_\theta$ maps it to a candidate set with confidences:
\begin{align}
(\mathcal{T}, \mathbf{p}) &= f_\theta(s), \\
\mathcal{T} &= \{\tau^{(1)},\ldots,\tau^{(K)}\},\quad
\mathbf{p} = (p^{(1)}, \ldots, p^{(K)}).\nonumber
\end{align}
The goal is to select $\hat{\tau} \in \mathcal{T}$ that is preferable under a prioritized rulebook. Standard geometric metrics---ADE, FDE, minADE, minFDE, MissRate ($\delta{=}2.0$\,m)---characterize candidate-set \emph{coverage} but not the \emph{acceptability} of the selected mode.
 
Best-of-$K$ minADE on WOMD \texttt{validation\_interactive} is strongly heavy-tailed: 0.46, 1.53, 2.16, and 3.62\,m at the 50th, 90th, 95th, and 99th percentiles---nearly $8\times$ from median to 99th (Fig.~\ref{fig:percentile_analysis}). In the tail, confidence-only selection is brittle because mode probabilities are poorly calibrated under rare interactions, while rule constraints (collision margins, signal compliance) remain decisive. This motivates selection grounded in traffic rules rather than likelihood alone.
 
%-------------------------------------------------------------------------------
\subsection{Traffic Rule Taxonomy}
\label{subsec:rule_taxonomy}
%-------------------------------------------------------------------------------
 
Let $\mathcal{R}$ denote a catalog of $R{=}28$ traffic rules, each assigned to one of four tiers:
\begin{equation}
\text{Safety} \succ \text{Legal} \succ \text{Road} \succ \text{Comfort}.
\end{equation}
Two applicability indicators are used throughout:
\begin{itemize}[leftmargin=1.5em]
\item \textbf{Oracle} $a_r(s) \in \{0,1\}$: ground-truth label from map/scenario metadata, used \emph{only} for post-hoc auditing (Protocol~B).
\item \textbf{Predicted} $\hat{a}_r(s) \in \{0,1\}$: thresholded output of the learned applicability head, used in the selection objective.
\end{itemize}
The gap between $a_r$ and $\hat{a}_r$ is a key source of approximation error. Each applicable rule produces a non-negative severity $V_r(\tau;s) \ge 0$. Because rules have heterogeneous units, we normalize to a common $[0,1]$ scale. The canonical instantiation (used in all experiments) is the exponential cost map from Section~\ref{Sec:Rule_Evaluation_And_Data_Pipeline}:
\begin{align}
\label{eq:normalization_contract}
\bar{V}_r(\tau;s) &= c_r(V_r(\tau;s)) \\
&= 1 - \exp\!\bigl(-\kappa_r \max(0, V_r(\tau;s))\bigr) \;\in\; [0,1],\nonumber
\end{align}
where $\kappa_r > 0$ is a per-rule softness parameter (tabulated in Table~\ref{tab:proxy_constants}). The linear clamp $\mathrm{clamp}(V_r/\alpha_r,0,1)$ satisfies the same $[0,1]$ contract and may be substituted for $c_r$ without affecting any downstream proposition; Propositions~\ref{thm:order_invariance}--\ref{thm:scalarization} require only $\bar{V}_r \in [0,1]$, which both normalizations guarantee.

Tier scores aggregate normalized violations using predicted applicability:
\begin{equation}
\label{eq:tier_score}
S_\ell(\tau;s) = \sum_{r:\,\ell(r)=\ell} \tilde{w}_r\,\hat{a}_r(s)\,\bar{V}_r(\tau;s),
\end{equation}
with $\tilde{w}_r \ge 0$ and $\sum_{r:\ell(r)=\ell}\tilde{w}_r = 1$ within each tier (default: uniform).  This normalization ensures $S_\ell \in [0,1]$ regardless of the number of rules per tier---a property required by Proposition~\ref{thm:scalarization}.  We collect the four scores into $\mathbf{S}(\tau;s) = [S_0, S_1, S_2, S_3]$.
 
\begin{definition}[Tolerance-Based Lexicographic Selection]
\phantomsection
\label{def:lex_selection}
Given candidates $\mathcal{T}$, tier score vectors $\mathbf{S}^{(k)}$, confidences $p^{(k)}$, and tolerances $(\varepsilon_0,\ldots,\varepsilon_3)$:
\begin{enumerate}
\item Initialize $\mathcal{C}_0 = \{1,\ldots,K\}$.
\item For each tier $\ell = 0,1,2,3$: compute $S_\ell^* = \min_{k \in \mathcal{C}_\ell} S_\ell^{(k)}$ and retain $\mathcal{C}_{\ell+1} = \{k \in \mathcal{C}_\ell : S_\ell^{(k)} \le S_\ell^* + \varepsilon_\ell\}$.
\item If $|\mathcal{C}_4| > 1$, break ties by selecting $k^* = \arg\max_{k \in \mathcal{C}_4} p^{(k)}$ (highest confidence among tier survivors).
\item If confidence is also exactly tied, return the candidate with the lowest index.
\end{enumerate}
\end{definition}

\noindent\textbf{Note on confidence usage.} Confidence $p^{(k)}$ does not enter compliance-based tier filtering (steps~1--2). It is used only in step~3, among candidates that are compliance-equivalent within tolerance. When all tier scores are zero---a majority case on WOMD---the lexicographic selector falls through to the confidence tiebreaker, whereas weighted-sum resolves tied scalar scores by its $\arg\min$ convention. These compliance-equivalent cases account for the selADE gap between the two selectors in Section~\ref{Sec:Accuracy_SectionRule-Compliance_Behavior_and_Efficiency}.
 
The definition is deliberately \emph{procedural}: the tolerance relation is not transitive, so pruning depends on the current pool's global minimum rather than pairwise comparisons. Operationally, the procedure is deterministic and respects tier priority.
 
%-------------------------------------------------------------------------------
\subsection{Formal Guarantees}
\label{subsec:formal_guarantees}
%-------------------------------------------------------------------------------
 
Three properties establish the selector's correctness: \textit{Order Invariance}, \textit{Infeasibility Fallback}, and \textit{Scalarization Correctness}:
 
\begin{proposition}[Order Invariance]
\label{thm:order_invariance}
The selector's output is invariant to candidate enumeration order, except in the degenerate case where two or more survivors share both identical tier scores and identical confidence; in that case, the lowest-index fallback (Def.~\ref{def:lex_selection}, step~4) is deterministic given the enumeration.
\end{proposition}
\begin{proof}
At each stage $\ell$, $C_{\ell+1}{=}\{k\in C_\ell : S_\ell^{(k)}\le\min_{j\in C_\ell}S_\ell^{(j)}+\varepsilon_\ell\}$ is a function of the multiset of tier scores in $C_\ell$, hence order-invariant; induction over $\ell{=}0,\ldots,3$ gives an order-invariant $C_4$. The confidence tiebreak $\arg\max_{k\in C_4} p^{(k)}$ is likewise a set operation. Order dependence enters only through step~4, which fires only when survivors share identical tier scores \emph{and} identical confidence---probability zero for softmax confidences over distinct logits at generic floating-point precision.
\end{proof}
 
\begin{remark}[Infeasibility Fallback]
\label{thm:infeasibility}
If all candidates have non-zero Safety-tier scores, the selector returns the least-violating candidate under the full lexicographic procedure: Tier~0 retains the minimum-Safety candidates within tolerance; lower tiers refine that survivor set. When the selected candidate has $S_0 > 0$, the system emits an explicit \texttt{infeasible} flag for downstream fail-safe logic. This flag reflects the proxy-evaluated Safety tier and is not, by itself, a formal collision certificate under the full non-differentiable evaluator; post-selection full-evaluator auditing is therefore recommended for deployment safety monitoring. On the WOMD validation benchmark at $K{=}6$, 96.8\% of scenarios contain at least one candidate with $S_0 = 0$; the infeasibility rate is approximately 3.2\%. Candidate policies for infeasible scenarios include conservative deceleration, a planner replan request, or an emergency-brake fallback; selection among these is a deployment-level decision outside the scope of this paper.
\end{remark}
 
\begin{proposition}[Scalarization Correctness]
\label{thm:scalarization}
For any resolution parameter $\varepsilon > 0$, the scalar score $\mathrm{Score}^{(k)} = \sum_{\ell=0}^{3} B^{4-\ell} S_\ell^{(k)}$ with base $B \ge \lceil 1/\varepsilon \rceil + 1$ preserves lexicographic ordering for every pair $(\tau^{(a)}, \tau^{(b)})$ whose first differing tier exhibits a score gap $\delta \ge \varepsilon$.  In particular, with $\varepsilon = 10^{-3}$ (the operating tier tolerance) and $B = 1001$, the scalarization is order-preserving on all pairs separated by at least one tolerance unit at the highest disagreeing tier.
\end{proposition}
\begin{proof}
Let $\ell^*$ be the least index with $S_{\ell^*}^{(a)}{<}S_{\ell^*}^{(b)}$, $\delta{=}S_{\ell^*}^{(b)}{-}S_{\ell^*}^{(a)}{>}0$. Using $S_\ell^{(k)}\in[0,1]$, the lower-tier perturbation is bounded by $\sum_{\ell>\ell^*}B^{4-\ell}<B^{4-\ell^*}/(B{-}1)$, so $\delta\ge 1/(B{-}1)$ suffices for $\mathrm{Score}^{(b)}{>}\mathrm{Score}^{(a)}$. Choosing $B\ge\lceil 1/\varepsilon\rceil{+}1$ gives $1/(B{-}1)\le\varepsilon$, so any $\delta\ge\varepsilon$ preserves order. With $\varepsilon{=}10^{-3}$, $B{=}1001$ gives threshold $10^{-3}$. Ties are resolved by the confidence tiebreaker (Def.~\ref{def:lex_selection}, step~3).
\end{proof}
 
These propositions are the selector-side basis for the paper's structural claim. Lexicographic ranking \emph{\textbf{cannot}} trade a higher-priority violation for a lower-priority gain, because tier filtering is applied \emph{\textbf{before}} lower tiers are consulted; no reparameterization changes this. A weighted-sum scalarization can always be reparameterized to allow such trade-offs, even if its current weights avoid them on a specific benchmark.
 
The formal problem statement is:

\begin{problem}[Rule-Aware Trajectory Selection]
\label{prob:selection}
The system is specified by
$\langle f_\theta,\; \mathcal{R},\; \hat{a},\; \bar{V},\; \mathsf{Select}\rangle$,
where $f_\theta$ generates candidates, $\mathcal{R}$ is the tiered rule catalog, $\hat{a}$ is the applicability policy, $\bar{V}$ is the normalized evaluation layer, and $\mathsf{Select}$ is the lexicographic selector.  The objective is:
\begin{equation}
\hat{\tau} = \mathsf{Select}\!\big(\mathcal{T},\;\mathbf{S}(\cdot\,;s),\;\mathbf{p},\;\boldsymbol{\varepsilon}\big).
\end{equation}
\end{problem}
 
This tuple separates generation from selection and exposes the two sources of approximation: applicability prediction and proxy evaluation. The architecture appears in Section~\ref{Sec:RECTOR}; empirical consequences appear in Sections~\ref{Sec:Experiments} and~\ref{Sec:Accuracy_SectionRule-Compliance_Behavior_and_Efficiency}.

\section{RECTOR Architecture}
\label{Sec:RECTOR}
 
\begin{table}[t]
\centering
\small
\caption{Parameter Breakdown (CVAE Generator + RECTOR Selection Layer)}
\label{tab:param_breakdown}
\begin{tabular}{lrr}
\toprule
\textbf{Component} & \textbf{Parameters} & \textbf{Fraction (\%)} \\
\midrule
Scene Encoder & 324,288 & 3.7 \\
Trajectory Decoder & 5,182,165 & 58.7 \\
Applicability Head + Scorer & 3,327,289 & 37.6 \\
\midrule
\textbf{Total} & \textbf{8,833,742} & \textbf{100.0} \\
\bottomrule
\end{tabular}
\end{table}

% \begin{algorithm}[ht]
% \caption{Lexicographic Trajectory Selection}
% \label{alg:lexicographic}
% \begin{algorithmic}[1]
% \State \textbf{Input:} Tier scores $\{\mathbf{S}^{(k)}\}_{k=1}^K$, confidences $\{p^{(k)}\}_{k=1}^K$, tolerances $\{\varepsilon_\ell\}$
% \State $\mathcal{C} \leftarrow \{1,\ldots,K\}$
% \For{$\ell = 0$ \textbf{to} $3$}
%     \State $S^* \leftarrow \min_{k \in \mathcal{C}} S_\ell^{(k)}$
%     \State $\mathcal{C} \leftarrow \{k \in \mathcal{C} : S_\ell^{(k)} \le S^* + \varepsilon_\ell\}$
%     \If{$|\mathcal{C}| = 1$} \textbf{break} \EndIf
% \EndFor
% \State \Return $\arg\max_{k \in \mathcal{C}} p^{(k)}$ \Comment{Confidence tie-break; then lowest index}
% \end{algorithmic}
% \end{algorithm} 

Given scenario $s{=}(\mathbf{x}_{\text{ego}},\mathbf{X}_{\text{agents}},\mathcal{M})$, RECTOR runs a three-stage pipeline:
\begin{enumerate}
\item \textbf{Scene encoding \& candidate generation:} An M2I-style encoder produces $\mathbf{e}_{\text{scene}} \in \mathbb{R}^{256}$; a Transformer--CVAE decoder generates $K{=}6$ trajectories with confidences.
\item \textbf{Applicability mechanism (tested reference, not deployment-default):} A 3.33\,M-parameter head predicts which of 28 rules apply. Section~\ref{subsubsec:applicability_negative} reports a negative result: a parameter-free always-on policy for Safety/Legal Pareto-dominates the learned head on cross-evaluated compliance \emph{and} accuracy, and is the recommended default; the learned head is retained as an ablation.
\item \textbf{Tiered lexicographic scorer:} Each candidate is scored on active rules via smooth proxies; the scorer returns the trajectory with the minimal violation vector under Safety~$\succ$~Legal~$\succ$~Road~$\succ$~Comfort.
\end{enumerate}

\noindent\small\textbf{Note:} The scorer has zero trainable parameters; the 3.33\,M count is the applicability MLP.
 
The encoder maps ego history, agent histories, and vectorized map context into a shared embedding $\mathbf{e}_{\text{scene}} \in \mathbb{R}^{D}$ ($D{=}256$) using pointwise encoding for agent and map tokens followed by Transformer fusion:
\begin{align}
\mathbf{h}_{\text{ego}} &= \text{Conv1D}(\mathbf{x}_{\text{ego}}, \text{kernel}{=}3) \in \mathbb{R}^{B \times 1 \times D}, \\
\mathbf{h}_{\text{lane},\ell} &= \max_p\;\text{MLP}(\mathcal{M}_{\text{lanes}}[\ell,p,:]) \in \mathbb{R}^{B \times 1 \times D}, \\
\mathbf{X}_{\text{tok}} &= \text{TransformerEnc}(\mathbf{h}_{\text{ego}} \oplus \mathbf{H}_{\text{agents}} \oplus \mathbf{H}_{\text{lanes}}), \\
\mathbf{e}_{\text{scene}} &= \text{QueryAttn}(\mathbf{q}_{\text{ego}}, \mathbf{X}_{\text{tok}}, \mathbf{X}_{\text{tok}}) \in \mathbb{R}^{B \times D},
\end{align}
where $\mathbf{q}_{\text{ego}}$ is a learned pooling query. The encoder is intentionally modest (Table~\ref{tab:param_breakdown}): compliance gains should come from rule-aware selection, not from inflating the backbone. Inputs are zero-padded to $N_{\text{pad}}{=}32$ agents and $L_{\text{pad}}{=}64$ lane segments (20 points each); the Transformer uses 3 layers, 8 heads, and a 1024-d FFN.
 
The decoder generates $K{=}6$ candidates over $T{=}50$ timesteps using mode-specific goal queries, a shared CVAE latent variable ($d_z{=}64$), and Transformer decoding ($d_{\mathrm{model}}{=}256$, 4 layers, 8 heads, 1024-d FFN, post-LN, GELU, dropout 0.1):
\begin{align}
\mathbf{g}_{\text{emb}}^{(k)} &= \text{GoalNet}(\mathbf{e}_{\text{scene}}, \mathbf{q}^{(k)}), \\
p(\mathbf{z} \mid \mathbf{e}_{\text{scene}}) &= \mathcal{N}(\boldsymbol{\mu}_{\text{prior}}, \text{diag}(\boldsymbol{\sigma}^2_{\text{prior}})).\nonumber
\end{align}
GoalNet is a two-layer MLP that maps the concatenated scene embedding and mode query ($[\mathbf{e}_{\text{scene}};\mathbf{q}^{(k)}] \in \mathbb{R}^{512}$) through a hidden layer of width 384 (GELU activation, dropout 0.1) to a 256-d goal embedding $\mathbf{g}_{\text{emb}}^{(k)}$. The mode queries $\{\mathbf{q}^{(k)}\}_{k=1}^{K}$, $\mathbf{q}^{(k)} \in \mathbb{R}^{256}$, are $K$ learned scene-independent embeddings, trained jointly in Stage~2. Diversity arises from the mode-specific query and a per-mode latent $\mathbf{z}^{(k)} \sim p(\mathbf{z}\mid\mathbf{e}_{\text{scene}})$ at inference; using the prior mean without distinct queries collapses the set. During training the posterior is $q(\mathbf{z}\mid \mathbf{e}_{\text{scene}},\mathbf{h}_{\text{gt}})=\mathcal{N}(\boldsymbol{\mu}_{\text{post}},\mathrm{diag}(\boldsymbol{\sigma}_{\text{post}}^2))$, with $\mathbf{h}_{\text{gt}}$ produced by a compact two-layer MLP (hidden 256, ReLU) over the flattened ground-truth future.
Each mode is decoded from $(\mathbf{e}_{\text{scene}}, \mathbf{g}_{\text{emb}}^{(k)}, \mathbf{z}^{(k)})$ into a state sequence $(x,y,\theta,v)$ in one non-autoregressive pass. Higher-order kinematics (acceleration, jerk) are obtained by finite differences during proxy evaluation. A two-layer regression head maps decoder outputs to waypoints; a linear confidence head produces mode probabilities via softmax. Per-mode latent sampling at inference is required to preserve diversity.

RECTOR does not apply a post-hoc kinematic feasibility filter. The generator is trained with a jerk-smoothness term ($\lambda_{\text{smooth}}{=}1.0$) that implicitly penalizes implausible trajectories, but no hard curvature or acceleration bound is enforced at inference. Acceleration and jerk proxies (L3.R0, L3.R1) implicitly penalize kinematic extremes at selection time; stronger generator-side guarantees are future work.

The applicability head is a \emph{tested reference component, not the deployment-default}: Sec.~\ref{subsubsec:applicability_negative} shows a parameter-free always-on policy for Safety/Legal Pareto-dominates it on cross-evaluated compliance \emph{and} accuracy. We retain its architecture as the most direct reference for future work on learned applicability. It uses a TierAwareBlock: each of $R{=}28$ rules has a learned query $\mathbf{q}_r \in \mathbb{R}^{256}$ that cross-attends over $\mathbf{e}_{\text{scene}}$, then passes through a two-layer MLP (256$\to$384$\to$1, GELU, sigmoid):
\begin{equation}
\tilde{a}_r = \sigma\!\left(\text{MLP}_r\bigl(\text{CrossAttn}(\mathbf{q}_r,\,\mathbf{e}_{\text{scene}})\bigr)\right) \in [0,1].
\end{equation}
Binary predictions $\hat{a}_r$ use per-tier thresholds $[0.05,0.15,0.30,0.50]$ (Safety$\to$Comfort) with tier-specific bias initialization $[{+}2.0,{+}1.0,0.0,{-}1.0]$; training uses focal BCE ($\gamma{=}2.0$) on oracle labels.

With the applicability mask $\hat{\mathbf{a}}$, the tiered scorer implements the lexicographic selection of Definition~\ref{def:lex_selection}:
\begin{equation}
S_\ell^{(k)}(s) = \sum_{r:\,\ell(r)=\ell} \tilde{w}_r\,\hat{a}_r(s) \cdot \bar{V}_{r,k}(s), \quad \ell \in \{0,1,2,3\}.
\end{equation}  
 
%-------------------------------------------------------------------------------
\subsection{Training Objective}
\label{subsec:training_objective}
%-------------------------------------------------------------------------------
 
RECTOR is trained in two stages. \textbf{Stage~1} (20 epochs): applicability head only, focal BCE on oracle labels. \textbf{Stage~2} (20+5 epochs): full end-to-end model with a composite loss:
\begin{align}
\mathcal{L} &= \lambda_{\text{WTA}}\mathcal{L}_{\text{WTA}} + \lambda_{\text{end}}\mathcal{L}_{\text{end}} + \lambda_{\text{KL}}\mathcal{L}_{\text{KL}} + \lambda_{\text{goal}}\mathcal{L}_{\text{goal}}  \\
&\quad + \lambda_{\text{smooth}}\mathcal{L}_{\text{smooth}} + \lambda_{\text{conf}}\mathcal{L}_{\text{conf}} + \lambda_{\text{app}}\mathcal{L}_{\text{app}} + \lambda_{\text{rule}}\mathcal{L}_{\text{rule}}.\nonumber
\end{align}
$\mathcal{L}_{\text{WTA}}$ is time-weighted Huber reconstruction on the best mode, $\mathcal{L}_{\text{end}}$ supervises the endpoint, $\mathcal{L}_{\text{KL}}$ regularizes the CVAE, $\mathcal{L}_{\text{goal}}$ aligns goal queries, $\mathcal{L}_{\text{smooth}}$ penalizes jerk, $\mathcal{L}_{\text{conf}}$ calibrates confidence, $\mathcal{L}_{\text{app}}$ supervises applicability, and $\mathcal{L}_{\text{rule}}$ provides weak proxy-based shaping ($\lambda_{\text{rule}}{=}1.0$). Its marginal contribution is not isolated here; compliance gains are primarily attributable to enforcement at the selection stage.

The KL term is $\mathcal{L}_{\text{KL}} = D_{\mathrm{KL}}\!\left(q(\mathbf{z}\mid \mathbf{e}_{\text{scene}},\mathbf{h}_{\text{gt}})\,\|\,p(\mathbf{z}\mid \mathbf{e}_{\text{scene}})\right)$. The rule-shaped term is the mean gated proxy cost across modes, $\mathcal{L}_{\text{rule}} = \frac{1}{K}\sum_{k=1}^{K}\sum_{\ell=0}^{3}\beta_{\ell}S_{\ell}^{(k)}$, with $\beta_{\ell}{=}1$ in our experiments. Averaging over modes can weakly discourage diversity; in practice, this is bounded by $\lambda_{\text{rule}} \ll \lambda_{\text{WTA}}$.
 
%-------------------------------------------------------------------------------
\subsection{Inference Procedure}
\label{subsec:inference}
%-------------------------------------------------------------------------------
\begin{algorithm}[ht]
\caption{RECTOR Inference}
\label{alg:inference}
\begin{algorithmic}[1]
\State \textbf{Input:} Scenario $s$, parameters $\theta$
\State \textbf{Output:} Selected trajectory index $k^*$
\State $\mathbf{e}_{\text{scene}} \leftarrow \text{SceneEncoder}(s)$
\For{$k = 1$ \textbf{to} $K$}
    \State $\mathbf{z}^{(k)} \sim p(\mathbf{z} \mid \mathbf{e}_{\text{scene}})$ \Comment{Independent sample per mode}
    \State $\tau^{(k)} \leftarrow \text{Decode}(\mathbf{e}_{\text{scene}}, \mathbf{g}_{\text{emb}}^{(k)}, \mathbf{z}^{(k)})$
\EndFor
\State $\mathbf{p} \leftarrow \text{softmax}(\text{logits})$
\State $\hat{\mathbf{a}} \leftarrow \mathbb{1}[\sigma(\text{ApplicabilityHead}(\mathbf{e}_{\text{scene}})) > \text{thresholds}]$
\For{each candidate $k$ and active rule $r$}
    \State $\bar{V}_{r,k} \leftarrow \text{RuleProxy}_r(\tau^{(k)}, s)$
\EndFor
\State Compute tier scores $S_\ell^{(k)}$ via Eq.~\eqref{eq:tier_score}
\State $k^* \leftarrow \text{LexicographicSelect}(\{\mathbf{S}^{(k)}\}, \mathbf{p}, \boldsymbol{\varepsilon})$
\State \Return $k^*$
\end{algorithmic}
\end{algorithm}
 
Inference is the forward pass without the posterior network or training losses. Selection is deterministic given $(\mathcal{T}, \mathbf{p}, \hat{\mathbf{a}})$, supporting regression testing under rule updates. Latency is dominated by generation; compliance evaluation is 1.4\,ms of the 7.3\,ms total.

% \begin{figure}[!t]
% \centering
% \includegraphics[width=0.85\columnwidth]{fig_latency_donut.png}
% \caption{Inference latency at batch size 128 on one A100 GPU (mean over 43{,}219 instances): 7.3\,ms/sample (137 samples/s). Generation dominates (5.9\,ms, 81\%); rule evaluation (1.0\,ms, 14\%) and lexicographic selection (0.4\,ms, 5\%) add 1.4\,ms total.}
% \label{fig:latency_donut}
% \end{figure}

\noindent\textbf{Plug-in contract.} The selection layer (stages 2--3) is decoupled from the generator. Any generator producing $\mathcal{T}{=}\{\tau^{(k)}\}_{k=1}^K$ with confidences $\mathbf{p}$ in world-frame at 10\,Hz (0.1\,s steps, 50 steps, $[x,y,\theta,v]$ state) can be substituted without changing the rule-evaluation or selection code. The output is a trajectory index $k^*$, the tier-score vector $\mathbf{S}^{(k^*)}$, and an \texttt{infeasible} flag when $S_0^{(k^*)} > 0$.

\section{Rule Evaluation and Data Pipeline}
\label{Sec:Rule_Evaluation_And_Data_Pipeline}
\begin{table*}[t]
\centering
\caption{Complete 28-rule catalog with applicability and proxy parameters. ``Applicable'' means support\_positive $> 0$ in WOMD v1.3.0 \texttt{validation\_interactive} (43,219 instances): 18 rules are invoked, and 10 are not observed. Rules without differentiable proxies are audit-only under Protocol~B.}
\label{tab:rule_catalog}
\label{tab:proxy_constants}
\scriptsize
\setlength{\tabcolsep}{2pt}
\resizebox{\textwidth}{!}{%
\begin{tabular}{lll p{3.3cm} c c c l p{4.6cm}}
\toprule
\textbf{Paper ID} & \textbf{Registry ID} & \textbf{Tier} & \textbf{Description} & \textbf{Applicable} & $\boldsymbol{\kappa_r}$ & \textbf{Threshold} & \textbf{Unit} & \textbf{Proxy Violation Semantics} \\
\midrule
L0.R0 & L0.R2 & Safety & Safe longitudinal distance & \checkmark & 2.0 & 2.0 & m & Longitudinal gap $< d_{\min}$ to lead vehicle \\
L0.R1 & L0.R3 & Safety & Safe lateral clearance & \checkmark & 2.0 & 0.5 & m & Lateral clearance $< d_{\text{lat}}$ to any agent \\
L0.R2 & L0.R4 & Safety & Crosswalk occupancy & \checkmark & 3.0 & 2.0 & m & Ego within crosswalk clearance zone \\
L0.R3 & L10.R1 & Safety & Collision avoidance (overlap) & \checkmark & 2.0 & 0.0 & m & OBB penetration depth (SAT) $> 0$ \\
L0.R4 & L10.R2 & Safety & VRU clearance buffer & --- & 2.0 & 1.0 & m & Distance to VRU $< d_{\text{VRU}}$ \\
\midrule
L1.R0 & L5.R1 & Legal & Traffic signal compliance & --- & 3.0 & 2.0 & m & Crossing stop line under red signal \\
L1.R1$^\dagger$ & L5.R2 & Legal & Priority / right-of-way & --- & --- & --- & --- & No differentiable proxy (audit-only) \\
L1.R2 & L7.R4 & Legal & Speed limit adherence & \checkmark & 2.0 & ---$^{\ddagger}$ & m/s & Speed $> v_{\text{limit}} + v_{\text{tol}}$ ($v_{\text{tol}}{=}$5\,mph) \\
L1.R3 & L8.R1 & Legal & Red-light stop compliance & \checkmark & 3.0 & 2.0 & m & Crossing stop line under red signal \\
L1.R4 & L8.R2 & Legal & Stop-sign compliance & \checkmark & 3.0 & 0.5 & m/s & Min speed in stop zone $> v_{\text{stop}}$ \\
L1.R5 & L8.R3 & Legal & Crosswalk yield to pedestrians & \checkmark & 3.0 & 5.0 & m & Ego within crosswalk yield radius \\
L1.R6 & L8.R5 & Legal & Wrong-way driving prevention & \checkmark & 2.0 & 2.356 & rad & Heading deviation $>$ 135$^\circ$ from lane direction \\
\midrule
L2.R0 & L3.R3 & Road & Drivable surface constraint & --- & 2.0 & 1.0 & m & Signed distance outside drivable surface \\
L2.R1 & L7.R3 & Road & Lane departure prevention & \checkmark & 2.0 & 1.8 & m & Lateral offset $>$ half lane width $+$ margin \\
\midrule
L3.R0 & L1.R1 & Comfort & Smooth longitudinal acceleration & \checkmark & 2.0 & 3.0 & m/s$^2$ & Longitudinal acceleration $> a_{\max}$ \\
L3.R1 & L1.R2 & Comfort & Smooth braking deceleration & \checkmark & 2.0 & 4.0 & m/s$^2$ & Braking deceleration $> a_{\text{brake}}$ \\
L3.R2 & L1.R3 & Comfort & Smooth lateral steering & --- & 2.0 & 0.5 & rad/s & Steering rate $> \omega_{\max}$ \\
L3.R3 & L1.R4 & Comfort & Speed consistency & --- & 2.0 & 2.0 & m/s & Speed change $> \Delta v_{\max}$ \\
L3.R4 & L1.R5 & Comfort & Jerk / lane-change smoothness & --- & 2.0 & 2.0 & m/s$^3$ & Jerk $> j_{\max}$ \\
L3.R5 & L4.R3 & Comfort & Left-turn gap acceptance & \checkmark & 2.0 & 4.0 & s & Left-turn TTC gap $< t_{\min}$ \\
L3.R6$^\dagger$ & L5.R3 & Comfort & Parking-zone violation & --- & --- & --- & --- & No differentiable proxy (audit-only) \\
L3.R7$^\dagger$ & L5.R4 & Comfort & School-zone speed compliance & --- & --- & --- & --- & No differentiable proxy (audit-only) \\
L3.R8$^\dagger$ & L5.R5 & Comfort & Construction-zone compliance & \checkmark & --- & --- & --- & No differentiable proxy (audit-only) \\
L3.R9 & L6.R1 & Comfort & Cooperative lane change & --- & 2.0 & 3.0 & s & Lane-change TTC $< t_{\min}$ \\
L3.R10 & L6.R2 & Comfort & Safe following distance & \checkmark & 2.0 & 2.0 & s & Following time $< t_{\text{follow}}$ \\
L3.R11 & L6.R3 & Comfort & Intersection negotiation & \checkmark & 2.0 & 3.0 & s & Intersection TTC $< t_{\min}$ \\
L3.R12 & L6.R4 & Comfort & Pedestrian interaction & \checkmark & 2.0 & 3.0 & m & Pedestrian proximity $< d_{\text{VRU}}$ \\
L3.R13 & L6.R5 & Comfort & Cyclist interaction & \checkmark & 2.0 & 3.0 & m & Cyclist proximity $< d_{\text{VRU}}$ \\
\bottomrule
\end{tabular}
}
\vspace{1mm}
\caption*{\footnotesize \textbf{Applicable:} \checkmark~= rule was invoked (support\_positive $> 0$) in WOMD v1.3.0; ---~= never invoked in 43,219 evaluation instances. $\kappa_r$: softness parameter in the exponential cost function (Eq.~\ref{eq:exponential_cost}). Higher $\kappa_r$ leads to faster saturation. Signal-related proxies (L1.R0, L1.R3--R5, L0.R2) use $\kappa_r{=}3.0$; all others use $\kappa_r{=}2.0$. For rules where the threshold is inherent to the violation definition (e.g., speed limit), ``---'' indicates a scenario-dependent threshold. L1.R1, L3.R6--L3.R8 have no proxy and remain audit-only. Additional penalty scaling: yellow-light violations $\times 0.3$, flashing-red $\times 0.5$. $^{\ddagger}$L1.R2 (speed limit): the per-scenario threshold is the WOMD lane-level speed limit when present. About 42\% of \texttt{validation\_interactive} lane segments lack an explicit speed-limit annotation; for those segments we fall back to a default of 25\,mph (urban) or 35\,mph (where the segment carries a freeway/ramp class label), which matches the U.S.\ statutory default for unposted urban roadways. Sensitivity at $\pm 5$\,mph defaults changes L1.R2 violation rate by $<0.4$~pp on Protocol~B.}
\vspace{-10pt}
\end{table*}

\begin{table}[ht]
\centering
\caption{Training vs. evaluation setup. Training uses smooth proxies for 24/28 rules for gradient learning; evaluation uses the full 28-rule Waymo framework with complete map context.}
\label{tab:training_eval_comparison}
\begin{tabular}{lcc}
\toprule
\textbf{Aspect} & \textbf{Training} & \textbf{Evaluation} \\
\midrule
Rules evaluated & 24 (proxied) & 28 (all) \\
Applicability source & Predicted ($\hat{a}_r$) & Oracle ($a_r^*$) \\
Map context & Partial & Full \\
Gradient flow & Yes & No \\
Speed & Fast ($\sim$1\,ms) & Moderate ($\sim$5\,ms) \\
\bottomrule
\end{tabular}
\vspace{-10pt}
\end{table}

\begin{table}[ht]
\centering
\caption{Data augmentation strategies in the ego-centric frame. Agent dropout simulates occlusion, trajectory noise regularizes dynamics, and flips/rotations improve spatial invariance; rule-targeted augmentation adds controlled-violation samples to balance proxy supervision.}
\label{tab:data_augmentation}
\begin{tabular}{llp{4cm}}
\toprule
\textbf{Augmentation} & \textbf{Parameters} & \textbf{Purpose} \\
\midrule
Rotation & $\pm 15^\circ$ & Orientation invariance \\
Translation & $\pm 2.0\,\mathrm{m}$ & Position invariance \\
Scaling & $0.95$--$1.05$ & Scale robustness \\
Agent dropout & 10\% prob. & Occlusion handling \\
Trajectory noise & $\mathcal{N}(0, 0.1^2)\,\mathrm{m}$ & Smoothness regularisation \\
Time reversal & 50\% prob. & Temporal symmetry \\
Mirror flip & 50\% prob. & Left-right symmetry \\
\bottomrule
\end{tabular}
\end{table}

This section specifies the rule catalog and the proxy layer used for selection-time ranking. RECTOR evaluates 28 rules across four tiers, with differentiable proxies for 24. Throughout, Protocol~A and Protocol~B refer to the settings defined in Section~\ref{subsec:eval_protocols}.

%-------------------------------------------------------------------------------
\subsection{Rule Catalog Overview}
\label{subsec:rule_catalog}
%-------------------------------------------------------------------------------

The catalog has 28 rules in four tiers consistent with Section~\ref{Sec:Problem_Formulation}: Tier~0 (Safety, physical safety/collision), Tier~1 (Legal, traffic-law concepts in the WOMD annotation framework\footnote{See the Legal-tier disclaimer in Sec.~\ref{Sec:Introduction}. In this document "rule" or any similar term is used as a mathematical term for constraints; it does not have any legal implications.}), Tier~2 (Road, drivable-surface/lane keeping), Tier~3 (Comfort, ride quality and context policies). Paper IDs $\mathrm{L}\ell.\mathrm{R}i$ are tier-local; Registry IDs in Table~\ref{tab:rule_catalog} index the \texttt{waymo\_rule\_eval} evaluator.

Of the 28 rules, 18 are invoked at least once across 43{,}219 instances; 10 are never observed. Low- or zero-support rules provide no training signal, which explains the learned head's high FNRs on rare rules and motivates the always-on policy for Safety and Legal (Sec.~\ref{Sec:Accuracy_SectionRule-Compliance_Behavior_and_Efficiency}).

We use the WOMD rule-evaluation framework as the authoritative auditor (consistent semantics, oracle applicability labels, stable regression targets). The full evaluator is non-differentiable, so RECTOR uses smooth proxies for optional rule-shaped training losses and stable selection-time ranking. Protocol~A evaluates only the proxy subset with predicted applicability; Protocol~B evaluates the full catalog with oracle applicability.

\subsubsection{Proxy Coverage}

\textbf{24 of 28 rules have smooth proxies} (86\% coverage). All Tier~0 rules are covered, so the highest-priority constraints always contribute a training signal. Table~\ref{tab:training_eval_comparison} summarizes training vs.\ evaluation.

The four unproxied rules (L1.R1, L3.R6, L3.R7, L3.R8) split into two categories: one requires multi-agent intent reasoning; the other three depend on map annotations that are unreliable in WOMD v1.3.0. At selection time, they contribute zero cost ($\hat{a}_r \cdot \bar{V}_r \equiv 0$) regardless of applicability, so their violations are invisible to the selector and only accumulate in Protocol~B audits. Notably, L3.R8 (construction-zone compliance) is marked applicable in WOMD but has no proxy, so it adds to Protocol~B counts while remaining invisible to the selector.

\subsubsection{Normalization of Proxy Severities}

Proxy outputs have heterogeneous units. Each proxy yields a raw severity $V_r(\tau;s)\ge 0$ and a normalized severity $\bar{V}_r(\tau;s)\in[0,1]$:
\begin{equation}
\bar{V}_r(\tau;s)=\min\!\left(1,\,\frac{V_r(\tau;s)}{\alpha_r}\right),
\end{equation}
with $\alpha_r{>}0$ from Section~\ref{Sec:Problem_Formulation}. This linear form is a proof-compatible template; experiments use the exponential form from Eq.~\eqref{eq:normalization_contract}, repeated for convenience:
\begin{equation}
c_r(v) = 1 - \exp\!\bigl(-\kappa_r \cdot \max(0, v)\bigr),
\label{eq:exponential_cost}
\end{equation}
where $\kappa_r > 0$ controls saturation. Both forms satisfy the $[0,1]$ contract required by Proposition~\ref{thm:scalarization}; reported results use the exponential form.

%-------------------------------------------------------------------------------
\subsection{Proxy Implementation Details}
\label{subsec:proxy_details}
%-------------------------------------------------------------------------------

We report one representative proxy per tier; the rest follow the same hinge- or event-based pattern with parameters in Table~\ref{tab:proxy_constants}. All run at 10\,Hz from decoded trajectories.

\textbf{Safety (L0.R3, collision avoidance):} $V=\sum_{t,j}\min(p_{\text{long}}(t,j),p_{\text{lat}}(t,j))$, with SAT-style overlap depths along the vehicle axes. \textbf{Legal (L1.R3, red-light stop):} $V=\sum_t r(t)\,\sigma_\alpha(\mathrm{cross}(t){-}\mathrm{cross}(t{-}1))$, $r(t)$ a soft red-signal indicator. \textbf{Road (L2.R1, lane departure):} $V=\sum_t\max(0,|d_t^{\text{lat}}|-w_{\text{lane}}/2-m)$. \textbf{Comfort (L3.R4, jerk):} $V=\sum_t\max(0,|j_t|-j_{\max})$.

%-------------------------------------------------------------------------------
\subsection{Data Pipeline}
\label{subsec:data_pipeline}
%-------------------------------------------------------------------------------

The same rule-evaluation machinery is used during proxy-based training and full-catalog auditing.

Dataset statistics and splits appear in Section~\ref{Sec:Experiments}. The evaluation set spans dense intersections, freeway merges, lane changes, and other interaction-rich geometries, requiring the catalog to cover both high-priority safety events and lower-priority comfort behaviors.

Raw TFRecord scenarios are converted to ego-centric tensors: coordinates are transformed to the ego frame, nearby agents and relevant lanes are selected, finite-difference kinematics are extracted, and map primitives (lane polylines, stop lines, crosswalks, drivable-surface polygons) are vectorized. Standard spatial and temporal augmentation is summarized in Table~\ref{tab:data_augmentation}.

We also generate synthetic collision, speed, lane-departure, and red-light violations to improve proxy supervision on rare but high-impact events.

\section{Datasets, Metrics, Baselines, and Implementation Details}
\label{Sec:Experiments}
This section fixes the evaluation scope. Unless stated otherwise, selector comparisons use the augmented WOMD \texttt{validation\_interactive} split (43{,}219 ego-scenario instances from 12{,}800 base scenarios), $K{=}6$, and report compliance on the \emph{selected} trajectory. $K{=}6$ follows the WOMD motion-prediction challenge convention; preliminary $K\in\{8,12\}$ runs gave no systematic compliance improvement at extra cost. Protocol~A is the operating selector benchmark (matches RECTOR's proxy stack); Protocol~B is the completeness reference (full 28-rule catalog with oracle applicability). We distinguish candidate-set quality from selected-trajectory quality, since RECTOR changes only the selection.
%-------------------------------------------------------------------------------
\subsection{Dataset}
\label{subsec:dataset}
%-------------------------------------------------------------------------------

We use the Waymo Open Motion Dataset (WOMD) v1.3.0~\cite{WOMD_20201}, which provides temporally aligned trajectories for the ego and surrounding agents, HD map context (lane topology, crosswalks, boundaries), and traffic-control state (signals, stop signs) required for both motion prediction and rule applicability/violation computation. Table~\ref{tab:womd_stats} summarizes key statistics. Fig.~\ref{fig:scenario_gallery} shows representative interactive scenarios spanning all four tiers.

\begin{figure*}
\centering

% ── Row 1: Intersection turn + Dense intersection ────────────────────────────
\begin{subfigure}[b]{0.22\textwidth}
  \includegraphics[width=\textwidth]{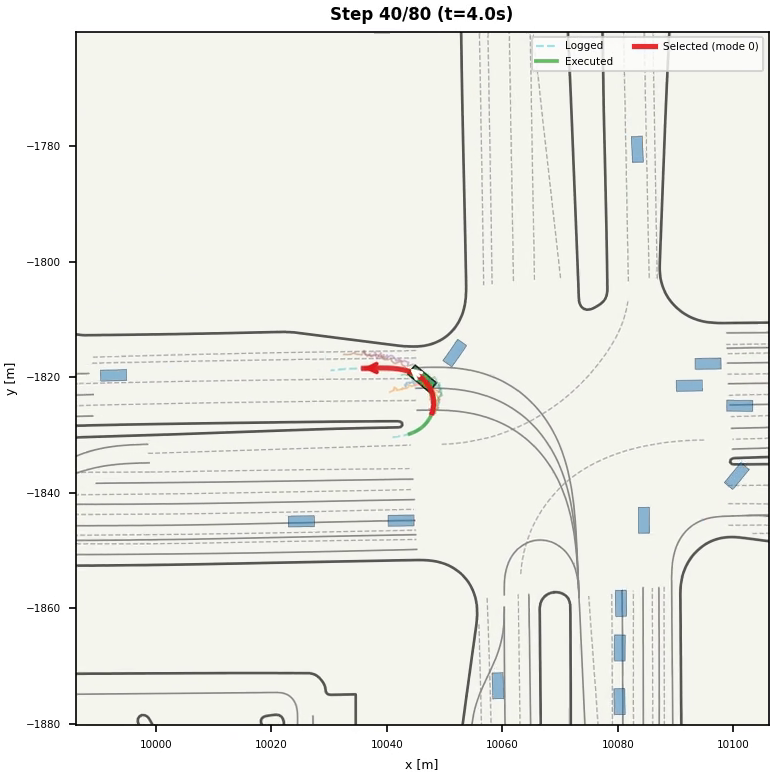}
  \caption{Intersection turn — scenario map}
\end{subfigure}\hfill
\begin{subfigure}[b]{0.22\textwidth}
  \includegraphics[width=\textwidth]{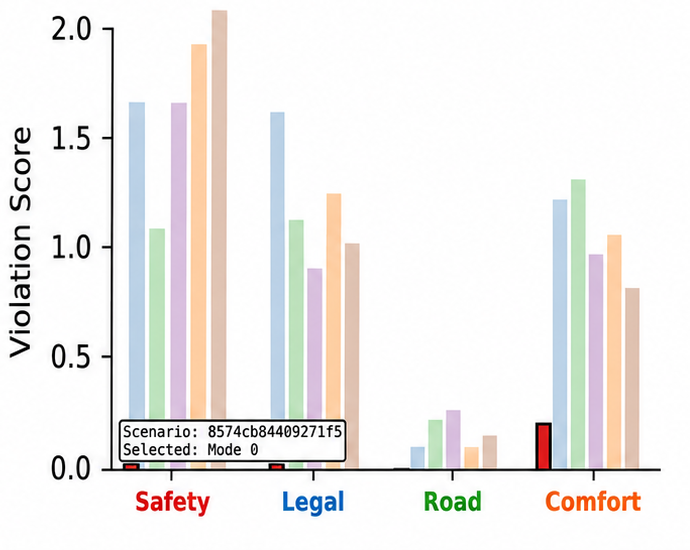}
  \caption{Intersection turn — per-tier violations}
\end{subfigure}\hfill
\begin{subfigure}[b]{0.22\textwidth}
  \includegraphics[width=\textwidth]{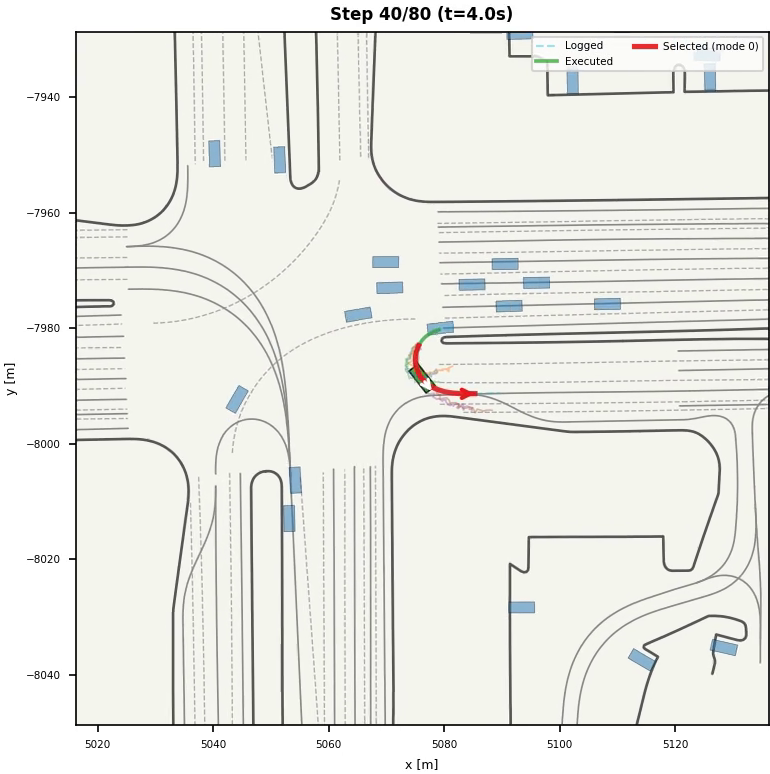}
  \caption{Dense intersection — scenario map}
\end{subfigure}\hfill
\begin{subfigure}[b]{0.22\textwidth}
  \includegraphics[width=\textwidth]{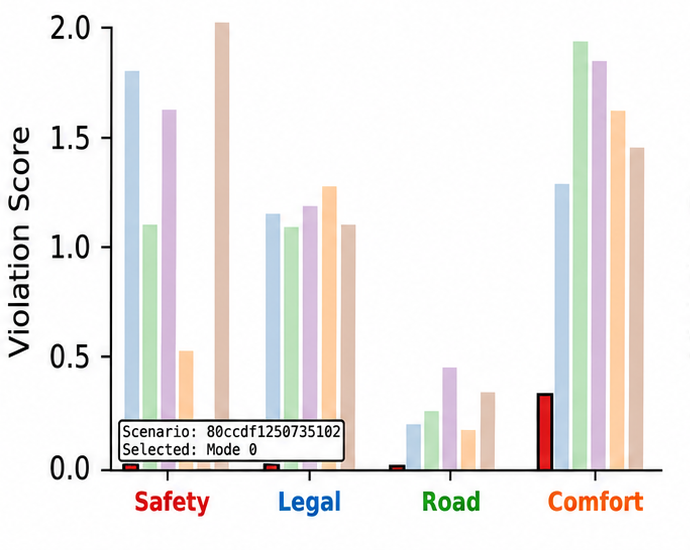}
  \caption{Dense intersection — per-tier violations}
\end{subfigure}

\hfill
\begin{subfigure}[b]{0.22\textwidth}
  \includegraphics[width=\textwidth]{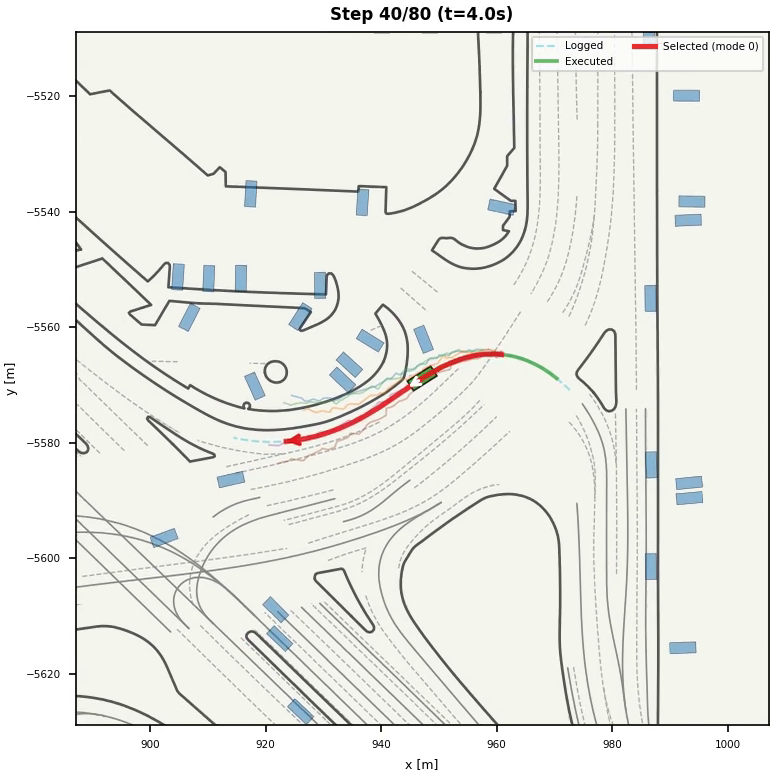}
  \caption{Exit-ramp merge — scenario map}
\end{subfigure}\hfill
\begin{subfigure}[b]{0.22\textwidth}
  \includegraphics[width=\textwidth]{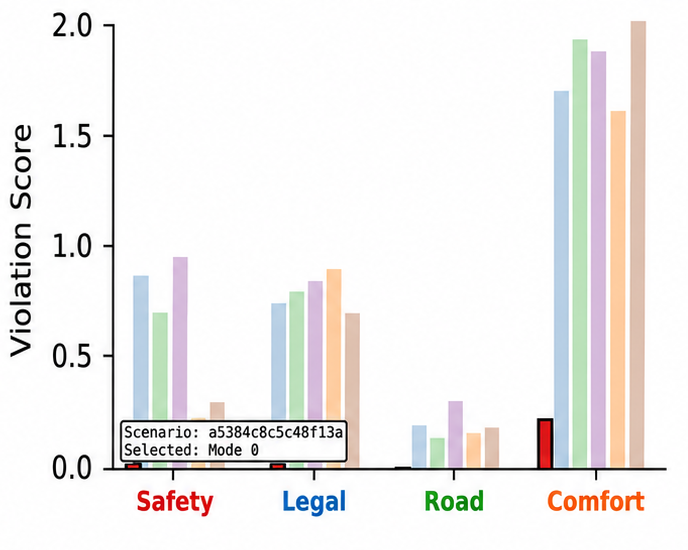}
  \caption{Exit-ramp merge — per-tier violations}
\end{subfigure}\hfill
\begin{subfigure}[b]{0.22\textwidth}
  \includegraphics[width=\textwidth]{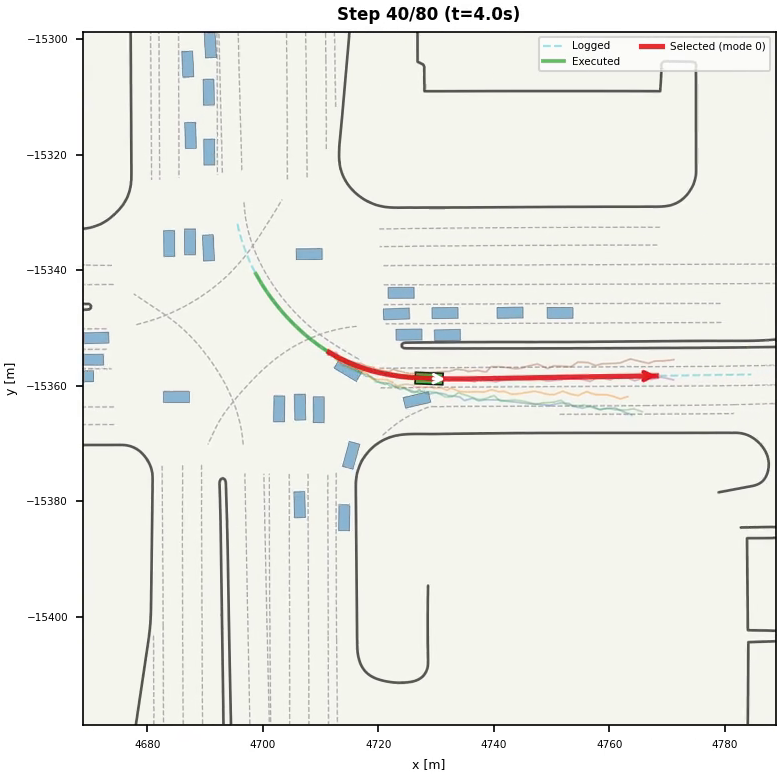}
  \caption{Lane change — scenario map}
\end{subfigure}\hfill
\begin{subfigure}[b]{0.22\textwidth}
  \includegraphics[width=\textwidth]{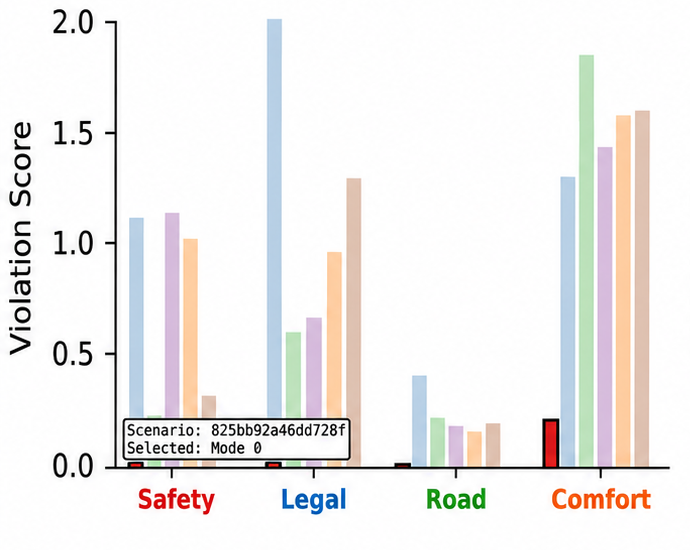}
  \caption{Lane change — per-tier violations}
\end{subfigure}

\caption{\textbf{Infrastructure-verification figure; not a selector comparison.} Four representative WOMD v1.3.0 \texttt{validation\_interactive} scenarios from the Waymax log-replay bridge (frames 35--40 of 8\,s). Left panels show BEV maps (blue: surrounding agents; red: RECTOR-selected trajectory; green: executed trajectory); right panels show per-tier violation scores across the six candidate modes. Because the \emph{MockLogReplayGenerator} returns the logged trajectory as $K{=}6$ near-identical candidates ($\sigma{=}0.001$), all three selectors produce identical selections by construction; the figure is included only to document that the Waymax execution path runs end-to-end and that the proxy scores trace correctly along the replayed rollout. Selector-disagreement BEV gallery on a reactive non-trivial generator is deferred to the journal-extension follow-up (Sec.~\ref{Sec:Accuracy_SectionRule-Compliance_Behavior_and_Efficiency}, ``Waymax Bridge''; Sec.~\ref{Sec:Conclusion}, next-step~(iii)).}
\label{fig:scenario_gallery}
\end{figure*}

\begin{table}[t]
\centering
\caption{Waymo Open Motion Dataset Statistics (Interactive Scenarios)}
\label{tab:womd_stats}
\begin{tabular}{ll}
\toprule
\textbf{Property} & \textbf{Value} \\
\midrule
\multicolumn{2}{l}{\textit{Dataset Overview}} \\
Total augmented scenarios & 159,924 \\
Training split & 102,040 augmented scenarios \\
Validation split & 12,800 base / 43,219 augmented \\
Test split & 14,665 augmented scenarios (withheld) \\
Interactive criterion & Exactly 2 objects\_of\_interest \\
Geographic regions & Phoenix, San Francisco, Mountain View \\
Traffic convention & Right-hand traffic (US) \\
Average agents per scenario & 8.2 \\
Maximum agents & 40+ \\
Object types & Vehicles, pedestrians, cyclists \\
Agent tracking consistency & Validity masks over the 9.1\,s window \\
\midrule
\multicolumn{2}{l}{\textit{Map Context}} \\
Map coverage & Full HD maps \\
Lane annotations & Complete lane topology \\
Traffic signals & Per-timestep state annotations \\
Stop signs & Location and orientation \\
Crosswalks & Polygon boundaries \\
Speed limits & Where available \\
\bottomrule
\end{tabular}
\end{table}

We follow WOMD's official interactive splits to avoid leakage. Training uses 102{,}040 augmented interactive scenarios; the selector benchmark uses the augmented \texttt{validation\_interactive} split (43{,}219 evaluation instances from 12{,}800 validation scenarios across 150 shards). Each augmented TFRecord wraps the original Scenario proto with per-rule applicability labels, violation flags, and severity scores from \texttt{waymo\_rule\_eval}; these labels are the oracle ground truth used in ablations (Table~\ref{tab:applicability_ablation}). The test split is withheld.

The $\sim$3.4$\times$ \emph{evaluation}-set augmentation (12{,}800 $\to$ 43{,}219) comes from ego-perspective permutation: each base scenario with two \texttt{objects\_of\_interest} is evaluated once per ego designation (including the original \texttt{sdc\_index}), giving 2--4 instances per base scenario depending on interaction type and validity-mask coverage. \emph{Geometric training augmentations in Table~\ref{tab:training_config} (random rotation, translation noise, agent dropout, time reversal, mirror flip) are applied only during training, not at evaluation.} Evaluation, therefore, uses original world-frame coordinates and timing. Higher absolute violation rates than single-perspective evaluation reflect harder interaction roles (yielding, merging) for which the predictor is less calibrated.

WOMD defines an \emph{interactive} scenario by exactly two \texttt{objects\_of\_interest}---the primary interacting entities---emphasizing cases where inter-agent coupling and priority reasoning are salient, an appropriate stress test for rule-aware selection. We use all interaction types in the validation split: v2v, v2p, v2c, and others.

Filtering is minimal and removes only degenerate cases: scenarios with fewer than 2 agents, ego speed consistently below $0.1\,\text{m/s}$, or incomplete map features. About 98\% of scenarios survive. The Legal-tier catalog reflects U.S.\ right-hand traffic norms in WOMD; other jurisdictions require re-parameterizing the catalog and re-evaluating proxy coverage.

%-------------------------------------------------------------------------------
\subsection{Evaluation Protocols}
\label{subsec:eval_protocols}
%-------------------------------------------------------------------------------

Compliance numbers are meaningful only when the evaluation scope is unambiguous. We use two protocols throughout.

\noindent\textbf{Protocol~A (operating selector benchmark).} The 24-rule proxy subset with predicted applicability $\hat{a}_r$. This is the like-for-like selector comparison since it matches the rule-aware re-ranking stack.

\noindent\textbf{Protocol~B (oracle-applicability proxy reference).} The full 28-rule catalog with oracle applicability $a_r^*$ and the same proxy-based violation stack. Used for full-rule accounting, applicability ablations, and conservative absolute compliance within the proxy-evaluator framework.

\noindent\textbf{Aggregation.} \emph{Total~Viol.\%} is the fraction of scenarios where the selected trajectory violates at least one rule (union over the rulebook). \emph{Per-tier~Viol.\%} is the fraction violating at least one rule in that tier. \emph{S+L~Viol.\%} is the fraction violating any Safety or Legal rule. Per-tier rates may sum above Total because one scenario can violate multiple tiers.

\noindent Protocol~B is more complete but its absolute rates need not exceed Protocol~A: oracle applicability re-exposes rules Protocol~A suppresses while removing false positives from the learned head. Protocol~B is therefore the stricter scope reference (oracle applicability + full catalog), rather than a separate, non-differentiable evaluator pass. For deployment-facing interpretation, Protocol~B with always-on Safety/Legal activation is the most conservative reading.

%-------------------------------------------------------------------------------
\subsection{Evaluation Metrics}
\label{subsec:metrics}
%-------------------------------------------------------------------------------

We report three metric groups: trajectory accuracy, rule compliance, and computational efficiency.

\subsubsection{Trajectory Accuracy} Standard ADE/FDE/minADE/minFDE/MissRate (Sec.~\ref{Sec:Problem_Formulation}) over $T{=}50$ at 10\,Hz with miss threshold $\delta{=}2.0$\,m, plus selected-trajectory variants $\text{selADE}{=}\text{ADE}(\hat\tau,\tau^{\mathrm{gt}})$ and $\text{selFDE}{=}\text{FDE}(\hat\tau,\tau^{\mathrm{gt}})$ that quantify decision quality rather than only candidate-set diversity.

\subsubsection{Rule Compliance} \emph{Tier-wise violation rate} $\text{ViolationRate}_\ell{=}\tfrac{1}{N}\sum_i\mathbb{1}[S_\ell(\tau_i;s_i){>}0]$; \emph{S+L} is the fraction violating any Safety or Legal rule; \emph{Total} is the union over the rulebook. Per-tier rates may sum to more than the total because a single scenario can span multiple tiers. Compliance Rate is the complement.

\subsubsection{Computational Efficiency} Mean/P95 inference latency, throughput, parameter counts, peak GPU memory.

%-------------------------------------------------------------------------------
\subsection{Baseline Methods}
\label{subsec:baseline_methods}
%-------------------------------------------------------------------------------

Comparisons serve two purposes: external baselines situate candidate-set accuracy among prior WOMD predictors, and internal baselines isolate the selector on a fixed candidate set.

\subsubsection{External Baselines} Candidate-set accuracy is compared against five published methods (LaneGCN~\cite{liang2020lanegcn}, DenseTNT~\cite{gu2021densetnt}, M2I~\cite{Sun_2022}, Wayformer~\cite{nayakanti2023wayformer}, MTR~\cite{shi2022motion}) in Table~\ref{tab:main_results}; comparisons are approximate due to implementation differences.

\subsubsection{Internal Selectors} All three operate on the same $K{=}6$ candidate set. \textbf{Confidence Only:} $\hat\tau{=}\arg\max_k p^{(k)}$. \textbf{Weighted Sum:} $\hat\tau{=}\arg\min_k \sum_r w_r\hat a_r\bar V_{r,k}$ with uniform $w_r{=}1$ and the default lowest-index tiebreak (compliance-equivalent all-zero cases need not match RECTOR's confidence tiebreak; the qualitative distinction is in Sec.~\ref{subsec:selection_results}). \textbf{Human GT (reference):} ground-truth trajectories under the identical Protocol~A pipeline; the \emph{imitation ceiling} for any system that perfectly replicates recorded behavior under the same stack. Its non-zero Protocol~A rate reflects that proxy thresholds were calibrated for prediction evaluation rather than imitation, and that human driving is occasionally aggressive relative to proxy margins. The 12.8\% Legal-tier rate under Protocol~B is a policy-relevant context, not a target.

%-------------------------------------------------------------------------------
\subsection{Implementation Details}
\label{subsec:implementation}
%-------------------------------------------------------------------------------

To enable reproducibility, we report the training and inference configuration below.

\subsubsection{Model Architecture} An interaction-aware encoder with a multimodal CVAE decoder, a rule-applicability head, and a tiered scorer (Table~\ref{tab:param_breakdown}); capacity beyond the motion backbone concentrates in applicability, the scorer is structural.

\subsubsection{Training Configuration} AdamW with OneCycleLR cosine and brief warm-up; hyperparameters, loss weights, augmentation, and hardware in Table~\ref{tab:training_config}.

\begin{table}[t]
\centering
\caption{Training hyper-parameters and Configuration}
\label{tab:training_config}
\begin{tabular}{ll}
\toprule
\textbf{hyper-parameter} & \textbf{Value} \\
\midrule
\multicolumn{2}{l}{\textit{Optimizer Settings}} \\
Optimizer & AdamW \\
Max learning rate & $3 \times 10^{-4}$ \\
Learning rate schedule & OneCycleLR (cosine) \\
Weight decay & $1 \times 10^{-4}$ \\
Gradient clipping & max norm = 1.0 \\
Batch size & 256 \\
Epochs (Stage~1 / Stage~2) & 20 / 20+5 \\
Precision & Mixed (AMP) \\
Random seed & 42 \\
\midrule
\multicolumn{2}{l}{\textit{Loss Weights}} \\
$\lambda_{\text{WTA}}$ (reconstruction) & 20.0 \\
$\lambda_{\text{end}}$ (endpoint) & 5.0 \\
$\lambda_{\text{KL}}$ (KL divergence) & 0.05 \\
$\lambda_{\text{goal}}$ (goal prediction) & 2.0 \\
$\lambda_{\text{smooth}}$ (smoothness) & 1.0 \\
$\lambda_{\text{conf}}$ (confidence) & 1.0 \\
$\lambda_{\text{app}}$ (applicability) & 1.0 \\
$\lambda_{\text{rule}}$ (proxy compliance) & 1.0 \\
\midrule
\multicolumn{2}{l}{\textit{Model Selection}} \\

Selection criterion & Best validation loss \\
Early stopping patience & 15 epochs \\
\midrule
\multicolumn{2}{l}{\textit{Data Augmentation}} \\
Random rotation & $\pm 15^\circ$ \\
Position noise & $\mathcal{N}(0, 0.1^2)$\,m \\
Agent dropout & 10\% \\
Time reversal$^*$ & 50\% probability \\
Mirror flip & 50\% probability \\
\bottomrule
\end{tabular}
\vspace{1mm}
\caption*{\footnotesize $^*$Time reversal is applied to spatial position sequences only, not to rule violation labels or temporal rule semantics. Rule applicability and violation labels are computed on the original (non-reversed) trajectory.  Training follows a two-stage procedure: Stage~1 trains only the applicability head (focal BCE, $\gamma{=}2.0$) on frozen encoder features for 20 epochs; Stage~2 fine-tunes the full model end-to-end with all eight loss terms for 20+5 epochs.  The loss weights in this table apply to Stage~2; Stage~1 uses $\mathcal{L}_{\text{app}}$ only.}
\vspace{-10pt}
\end{table}

\subsubsection{Loss Functions and Inference} The training objective is the flat weighted sum of the eight terms in Sec.~\ref{subsec:training_objective}; the model with the lowest composite validation loss is selected, and because the reconstruction term ($\lambda_{\text{WTA}}{=}20.0$) dominates, this closely tracks the best validation ADE. Training follows a two-stage procedure (Stage~1: an applicability head with focal BCE; Stage~2: a full end-to-end model). All selector comparisons use the final model selected on the 43{,}219-instance augmented validation set; the augmented-set minADE of 0.805\,m reflects evaluation across all ego perspectives. Inference runs in FP32 at batch size 128; latency, throughput, and memory are reported in Sec.~\ref{Sec:Accuracy_SectionRule-Compliance_Behavior_and_Efficiency}.

All experiments use Python~3.10, PyTorch~2.2, CUDA~12.1, Waymo Open Dataset SDK~v1.5.2 (the parsing library; the dataset itself is WOMD v1.3.0, as in Section~\ref{Sec:Experiments}), and a fixed seed of 42 with deterministic CUDA. Full dependency versions and the preprocessing pipeline accompany the code release.

\subsubsection{Code and Artifacts}
\label{subsec:code_release}

The full training pipeline, rule catalog with proxy implementations, applicability head, lexicographic selector, evaluation scripts, and the augmentation pipeline that produced the 43{,}219 evaluation instances will be released at \texttt{https://github.com/\textlangle{}anonymized\textrangle/RECTOR} (URL and Zenodo DOI to be activated upon acceptance; an anonymized snapshot is available to reviewers on request). The release includes seed configuration, deterministic CUDA flags, and the proxy/auditor concordance harness used in Section~\ref{Sec:Section_VIII_Verification_Invariants_Numerical_Stability_and_Integration_Tests}.

\subsubsection{Scope, Caveats, and Limitations of This Evaluation}
\label{subsec:eval_limitations}

All selector comparisons in Section~\ref{Sec:Accuracy_SectionRule-Compliance_Behavior_and_Efficiency} are bounded by the following methodological caveats. (i)~\textbf{Validation-split model selection.} The WOMD test split lacks the per-rule applicability/violation labels required by our augmented pipeline, so we select the final checkpoint by best validation composite loss and report on the same augmented \texttt{validation\_interactive} split; absolute generation and Protocol~A numbers therefore carry an unquantified optimistic bias, while Protocol~B selector deltas are insulated from learned-applicability bias. Held-out test-split reporting is an explicit next step (Section~\ref{Sec:Conclusion}). (ii)~\textbf{Statistical tests.} The 43{,}219 augmented instances are paired across selectors on identical candidate sets; multiple instances share base scenarios and map context, so significance values reflect paired-instance contrasts, not $N{=}43{,}219$ independent draws. A cluster-robust reading over $\sim$12{,}800 base scenarios yields the same selector ordering with wider intervals. (iii)~\textbf{Sensitivity scope.} Reported results use $\varepsilon{=}10^{-3}$, $K{=}6$, $w_r{=}1$; spot-checks at $\varepsilon\in\{10^{-4},10^{-2}\}$ and $K\in\{8,12\}$ do not change the Protocol~B selector ordering, and the full grid is in supplementary material. (iv)~\textbf{Protocol scope.} All compliance numbers are proxy-evaluator results under Protocol~A or Protocol~B and are not full Waymo-evaluator passes; concordance and the L0.R3 caveat are characterized in Section~\ref{Sec:Section_VIII_Verification_Invariants_Numerical_Stability_and_Integration_Tests}.

\section{Accuracy, Rule Compliance, Behavior, and Efficiency}
\label{Sec:Accuracy_SectionRule-Compliance_Behavior_and_Efficiency}
This section reports the main empirical findings on the augmented WOMD \texttt{validation\_interactive} split (43{,}219 instances, $K{=}6$). We lead with \textbf{Protocol~B} (full 28-rule catalog, oracle applicability) as the primary headline because it is insulated from the learned applicability head's calibration issues (Sec.~\ref{subsec:ablations_main}) and exercises all four tiers; Protocol~A (24 proxy rules, predicted applicability) is the operating-stack diagnostic. \emph{Best-of-$K$} accuracy measures candidate-set coverage; \emph{selected} accuracy measures the executed trajectory. All compliance numbers are proxy-evaluator results at their stated scope, not full-evaluator certificates (Sec.~\ref{subsec:eval_limitations}).

\paragraph*{Headline (Protocol~B)} Rule-aware selection reduces \textbf{Safety+Legal violations from 28.58\% to 20.42\%} (8.16~pp) and \textbf{Total from 40.32\% to 32.41\%} (7.91~pp) versus confidence-only on the same candidates (Table~\ref{tab:per_tier_oracle}); per-tier reductions hold on every tier. Under adversarial confidence corruption (Sec.~\ref{subsubsec:corruption}), confidence-only fails in 100\% of scenarios while rule-aware selection rejects the injection in $\sim$96\%. All Safety figures are proxy-evaluator rates; L0.R3 has 80.7\% FNR against the Waymo full evaluator (Sec.~\ref{Sec:Section_VIII_Verification_Invariants_Numerical_Stability_and_Integration_Tests}), so a post-selection full-evaluator audit is required for any deployment-facing reading.

We begin with candidate-set accuracy, which upper-bounds any selector. Table~\ref{tab:main_results} places RECTOR's generation among published methods: \textbf{minADE\,=\,0.805\,m}, \textbf{minFDE\,=\,1.513\,m}, \textbf{Miss\,=\,23.33\%}. The candidate set is strong enough for meaningful selection experiments; the rest of the section isolates the selector.

\begin{table}[t]
\centering
\caption{Candidate-set accuracy on WOMD \texttt{validation\_interactive} (43{,}219 augmented instances, $K{=}6$). minADE/minFDE are best-of-$K$ (FDE from the ADE-best mode), and Miss uses a 2.0\,m FDE threshold. RECTOR standard deviations are $\pm$0.828\,m (minADE) and $\pm$2.07\,m (minFDE). $^\dagger$Baselines are reported from original papers and may use different splits, augmentation, or $K$, so comparisons are approximate.}
\label{tab:main_results}
\begin{tabular}{lccc}
\toprule
\textbf{Method} & \textbf{minADE (m)} \color{green}{$\downarrow$} & \textbf{minFDE (m)} \color{green}{$\downarrow$} & \textbf{Miss (\%)} \color{green}{$\downarrow$} \\
\midrule
\multicolumn{4}{l}{\textit{Published results (reported from original papers$^\dagger$)}} \\
LaneGCN & 0.87 & 1.92 & 28.4 \\
DenseTNT & 0.73 & 1.55 & 20.1 \\
M2I & 0.75 & 1.63 & 22.8 \\
Wayformer & 0.68 & 1.42 & 18.9 \\
MTR & 0.70 & 1.48 & 19.5 \\
\midrule
\multicolumn{4}{l}{\textit{Reproduced in our pipeline}} \\
\textbf{RECTOR (Ours)} & \textbf{0.805} & \textbf{1.513} & \textbf{23.33} \\
\bottomrule
\end{tabular}
\end{table}

%-------------------------------------------------------------------------------
\subsection{Selection Strategy Comparison}
\label{subsec:selection_results}
%-------------------------------------------------------------------------------

This experiment isolates the paper's central claim: compliance can be improved \emph{by selection alone} when candidates are held fixed. We apply three strategies to the same candidate set $\mathcal{T}(s)$---Confidence Only ($\arg\max_k p^{(k)}$), Weighted Sum ($\arg\min_k \sum_r w_r \bar{V}_{r,k}$), and RECTOR (lexicographic)---on 43{,}219 instances. Table~\ref{tab:per_tier_oracle} is the primary view under Protocol~B (full catalog, oracle applicability); Table~\ref{tab:selection_strategy} is the operating-stack diagnostic under Protocol~A.

\begin{table*}[t]
\centering
\caption{\textbf{Protocol~A (operating-stack diagnostic):} 24 proxy rules with predicted applicability $\hat{a}_r$. Selector comparison on a fixed candidate set (43{,}219 scenarios). selADE/selFDE are selected-trajectory errors; S+L = Safety~$\cup$~Legal; Total counts scenarios with $\geq 1$ active-rule violation. Legal and Road are 0.00\% because the learned applicability head suppresses those tiers (a head-calibration artifact, not a selector property); the headline view is Table~\ref{tab:per_tier_oracle} under Protocol~B. Weighted-sum and lexicographic have identical binary compliance; the selADE gap arises because compliance-equivalent all-zero-score cases are resolved by RECTOR's confidence tiebreaker vs.\ the weighted-sum scalar $\arg\min$.}
\label{tab:selection_strategy}
\setlength{\tabcolsep}{4pt}
\small
\begin{tabular}{lcccccccc}
\toprule
\textbf{Strategy} & \textbf{selADE} & \textbf{selFDE} & \textbf{Safety} & \textbf{Legal} & \textbf{Road} & \textbf{Comfort} & \textbf{S+L} & \textbf{Total} \\
 & \textbf{(m)} & \textbf{(m)} & \textbf{(\%)} & \textbf{(\%)} & \textbf{(\%)} & \textbf{(\%)} & \textbf{(\%)} & \textbf{(\%)}\\
\midrule
Confidence Only & 2.953 & 6.270 & 68.39 & 0.00 & 0.00 & 3.43 & 68.39 & 69.96\\
Weighted Sum & 2.33 & 4.941 & 56.07 & 0.00 & 0.00 & 2.68 & 56.07 & 57.99\\
RECTOR (lex.) & 2.74 & 5.729 & \textbf{56.07} & \textbf{0.00} & \textbf{0.00} & \textbf{2.67} & \textbf{56.07} & \textbf{57.99}\\
\bottomrule
\end{tabular}
\end{table*}

\begin{figure}[!t]
\centering
\includegraphics[width=0.85\columnwidth]{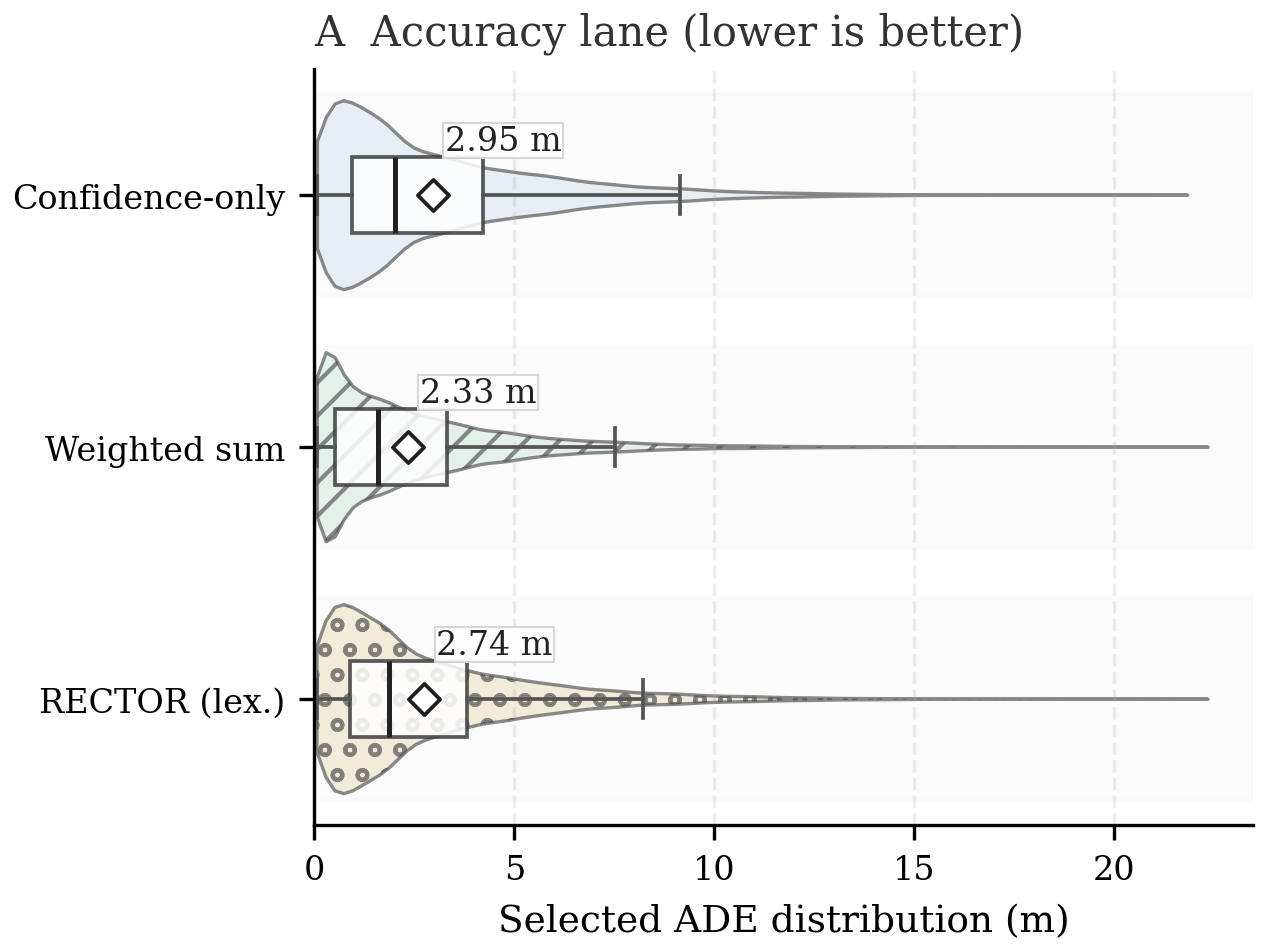}
\caption{Selected-trajectory accuracy (selADE) for the three selectors on the same WOMD \texttt{validation\_interactive} candidate set. Violin+box overlays summarize distribution and spread; mean selADE is 2.95\,m (Confidence-only), 2.33\,m (Weighted sum), and 2.74\,m (RECTOR lexicographic).}
\vspace{-10pt}
\label{fig:selector_accuracy}
\end{figure}

Figure~\ref{fig:selector_accuracy} and Table~\ref{tab:selection_strategy} summarize the operating-stack diagnostic. \textbf{Under Protocol~A, rule-aware selection cuts reported Safety violations by 12.32~pp, from 68.39\% (Confidence-only) to 56.07\%.} These Protocol~A absolute rates are inflated by Safety over-activation in the learned head and suppression of Legal/Road; the Protocol~B oracle baseline is 27.94\%. We therefore read Protocol~A as the delta induced by rule-aware reranking inside the deployed proxy configuration, not as a headline absolute. The deployed-stack value here is rule-aware over confidence-only; the structural separation between lexicographic and weighted-sum is addressed in Secs.~\ref{subsubsec:cannot_vs_does_not} and~\ref{subsubsec:corruption}.

Selected-trajectory endpoint error follows the selADE pattern: under Protocol~A, both rule-aware selectors improve selFDE over confidence-only (5.729\,m / 4.941\,m vs.\ 6.270\,m). The residual gap between lexicographic and weighted-sum stems from compliance-equivalent cases (a majority at $K{=}6$): when all four tier scores are zero, the lexicographic selector falls through to its confidence tiebreaker (Def.~\ref{def:lex_selection}, Step~3) while weighted-sum uses its own $\arg\min$.

Paired tests support the same conclusion. McNemar gives $p \approx 0$ for lexicographic vs.\ confidence-only on Safety (5{,}326 discordant pairs, all favoring rule-aware) and Total (5{,}175). Wilcoxon on selADE gives $p{=}3.95{\times}10^{-85}$ with a 0.216\,m mean improvement. Lex vs.\ WS gives McNemar $p{=}1.0$ (identical binary outcomes under Protocol~A); Wilcoxon on selADE favors WS (mean $\Delta$ 0.407\,m). A 250-config weight grid confirms WS matches lex binary compliance whenever $w_S \ge 5$; only fully flat weights produce ten extra Safety-violating scenarios ($+$0.023~pp). $p$-values are paired-instance on the augmented 43{,}219 set; multiple instances may share a base scenario.

\subsubsection{Cannot vs.\ Does Not: Why WS\,=\,Lex on This Benchmark}
\label{subsubsec:cannot_vs_does_not}

WS and lex produce equal per-tier binary compliance not because the selectors are equivalent but because WOMD rarely exposes cross-tier conflicts: only \textbf{0.40\% (171 scenarios)} contain a candidate that is strictly Pareto-conflicting across tiers. On the remaining 99.60\%, candidates are either fully zero-score (the lex confidence tiebreak fires) or one candidate dominates on every tier (any monotone scalarization picks the same one), so binary compliance matches by construction. This is a \emph{does-not}, not a \emph{cannot}: the lex guarantee is that \emph{no} weighted-sum reparameterization can trade higher-priority for lower-priority violations within $\varepsilon$, and the weight-grid sweep confirms matching requires $w_S \ge 5$. Distributions with denser cross-tier conflicts (heavier intersections, more frequent emergency yields) would empirically separate the two; reporting on such distributions is future work.

\subsubsection{Adversarial Confidence Corruption}
\label{subsubsec:corruption}

The single clearest separation between confidence-based and rule-aware selection appears under adversarial confidence corruption. We perturb the candidate set by injecting a Safety-violating mode and forcing its confidence to exceed that of all alternatives. Across three injection families (collision-prone, off-road, signal-violating), confidence-only picks the injected violator in \textbf{100\%} of scenarios; both rule-aware selectors reject it in \textbf{$\sim$96\%}. Combined with the structural lex guarantee (Sec.~\ref{subsubsec:cannot_vs_does_not}), this is robustness no scalar recalibration can match.

\begin{figure}[!t]
\centering
\includegraphics[width=\columnwidth]{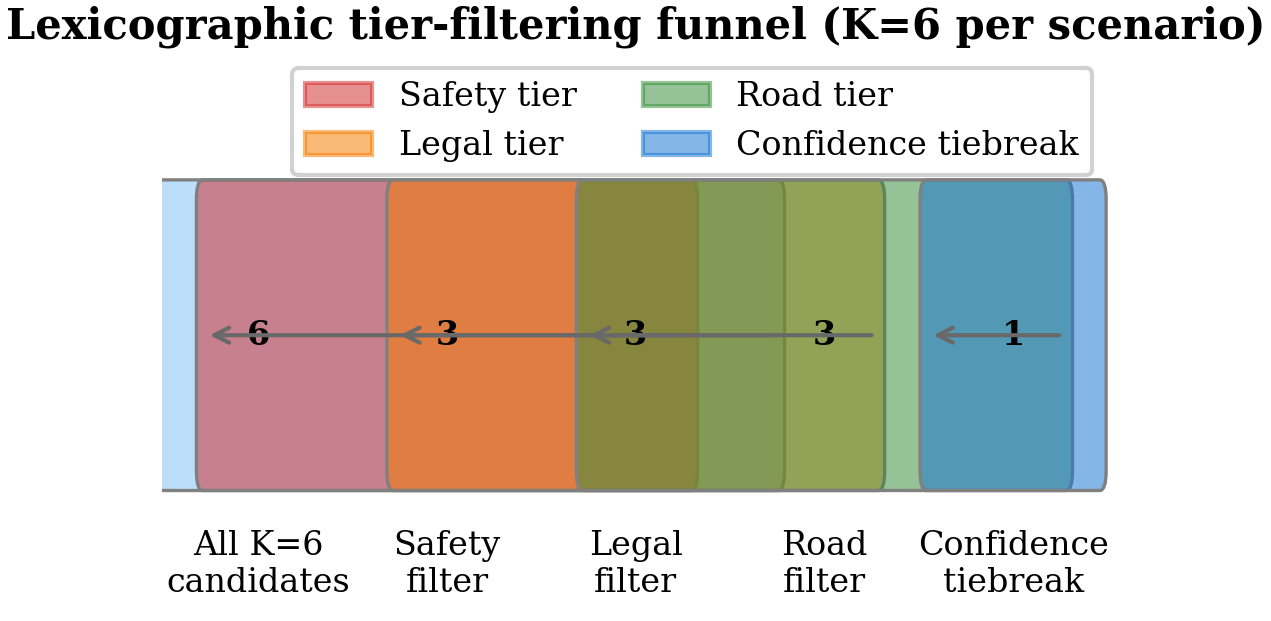}
\caption{Lexicographic tier-filtering funnel on WOMD \texttt{validation\_interactive} (43{,}219 scenarios). From $K{=}6$, each tier keeps candidates within $\varepsilon_\ell$ of the minimum tier score. The mean survivors are 5.77 (Tier~0), 5.67 (Tier~1), 5.67 (Tier~2), and 2.84 (Tier~3); then confidence selects one. The Tier~3 drop indicates Comfort is the main discriminator when Safety scores tie (96.8\% of scenarios have at least one $S_0{=}0$ candidate).}
\vspace{-10pt}
\label{fig:tier_funnel}
\end{figure}

Both rule-aware strategies also improve geometric accuracy versus confidence-only (selADE: 2.953\,m\,$\to$\,2.33\,m for WS, 2.74\,m for RECTOR), since safety-violating trajectories tend to be less accurate. Table~\ref{tab:per_tier_oracle} gives the corresponding Protocol~B oracle-applicability view on the same 43{,}219 scenarios. Under oracle applicability, all four tiers are non-zero, and rule-aware selection reduces violations across each tier relative to confidence-only.

\begin{table*}[t]
\centering
\caption{\textbf{Protocol~B (primary headline view):} full 28-rule proxy catalog with oracle applicability $a_r^*$. Selector comparison on the same 43{,}219-instance candidate set. Relative to confidence-only, both rule-aware selectors reduce violations across all tiers: Safety 8.15~pp, Legal 0.07~pp, Road 0.04~pp, Comfort about 2.4~pp, S+L 8.16~pp, Total 7.91~pp. Weighted-sum and lexicographic are equal at reported precision on binary per-tier compliance because cross-tier conflicts are rare in WOMD \texttt{validation\_interactive} (\emph{does not}, not \emph{cannot}; see Sec.~\ref{subsubsec:cannot_vs_does_not}); the structural lexicographic guarantee remains a separate property of the selector.}
\label{tab:per_tier_oracle}
\setlength{\tabcolsep}{4pt}
\small
\begin{tabular}{lcccccccc}
\toprule
\textbf{Strategy} & \textbf{selADE} & \textbf{selFDE} & \textbf{Safety} & \textbf{Legal} & \textbf{Road} & \textbf{Comfort} & \textbf{S+L} & \textbf{Total} \\
 & \textbf{(m)} & \textbf{(m)} & \textbf{(\%)} & \textbf{(\%)} & \textbf{(\%)} & \textbf{(\%)} & \textbf{(\%)} & \textbf{(\%)} \\
\midrule
Confidence Only & 2.953 & 6.270 & 27.94 & 0.82 & 1.60 & 14.80 & 28.58 & 40.32 \\
Weighted Sum & 2.370 & 5.060 & 19.79 & 0.76 & 1.55 & 12.40 & 20.42 & 32.41 \\
RECTOR (lex.) & 2.875 & 6.063 & \textbf{19.79} & \textbf{0.76} & \textbf{1.55} & \textbf{12.39} & \textbf{20.42} & \textbf{32.41} \\
\bottomrule
\end{tabular}
\vspace{1mm}
\end{table*}

\begin{figure}[!t]
\centering
\includegraphics[width=\columnwidth]{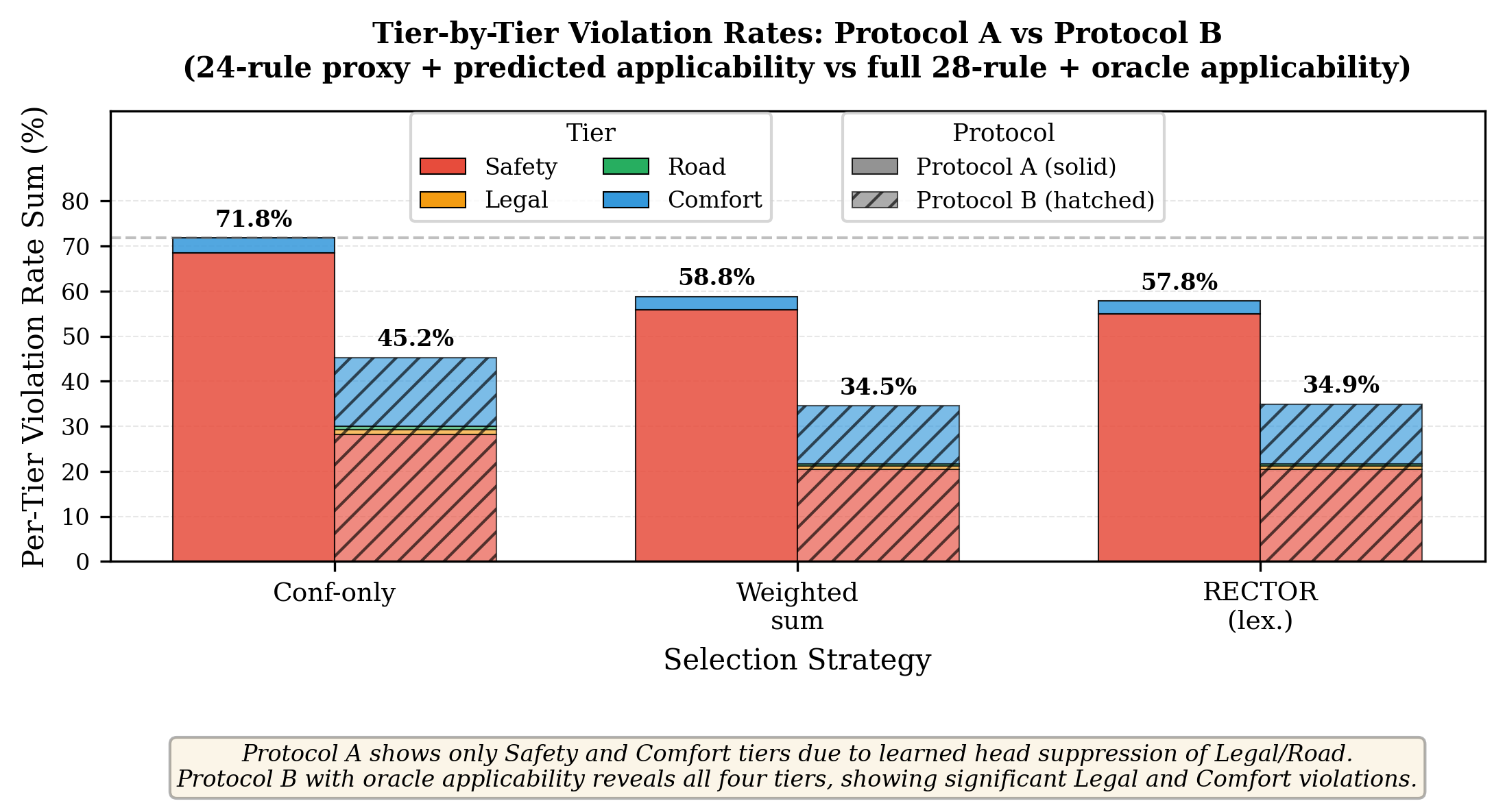}
\caption{Tier-wise violation rates for Protocol~A (solid; 24-rule subset with learned applicability) and Protocol~B (hatched; full 28-rule catalog with oracle applicability) across three selectors on 43{,}219 augmented validation instances. Bar height is the \emph{sum over tiers} (Safety+Legal+Road+Comfort), which is higher than Total Viol.\% in Tables~\ref{tab:selection_strategy} and~\ref{tab:per_tier_oracle} because Total counts the union over tiers. Protocol~A suppresses Legal and Road, while Protocol~B reveals all tiers. Per-tier sums are labeled above bars.}
\vspace{-5pt}
\label{fig:violation_stacked_bar}
\end{figure}

Per-rule heatmaps for Protocol~B are included in supplementary material; the main-text conclusions are supported by the aggregate and per-tier results in Tables~\ref{tab:selection_strategy} and~\ref{tab:per_tier_oracle}.

%-------------------------------------------------------------------------------
\subsection{Ablation Studies}
\label{subsec:ablations_main}
%-------------------------------------------------------------------------------

This subsection reports ablations under Protocol~B (full 28-rule proxy stack, oracle applicability) to isolate each architectural component with non-zero violations across all tiers.

To diagnose the learned applicability head, we analyze its per-rule F1 on the full 43{,}219-scenario split: F1 correlates strongly with oracle support ($\rho{=}0.981$), with rules having $>$5{,}000 positives reaching mean $\text{F}_1{=}0.77$ and low-support rules collapsing to $\text{F}_1{=}0.00$. Figure~\ref{fig:f1_vs_support} visualizes this; the head has severe false-negative bias for the 10 never-invoked rules. Table~\ref{tab:applicability_ablation} quantifies the impact: oracle applicability cuts Safety violations from 56.07\% to 19.79\% and Total from 57.99\% to 32.41\%, while restoring non-zero Legal and Road rates.

\begin{table}[t]
\centering
\caption{Applicability-source ablation (oracle vs. learned) under lexicographic selection on 43{,}219 scenarios. Oracle applicability changes selected trajectories, reducing Safety from 56.07\% to 19.79\% and Total from 57.99\% to 32.41\%, while re-exposing non-zero Legal and Road rates suppressed by the learned head.}
\label{tab:applicability_ablation}
\setlength{\tabcolsep}{3pt}
\begin{tabular}{lccccc}
\toprule
\textbf{Applicability} & \textbf{Total} & \textbf{Safety} & \textbf{Legal} & \textbf{Road} & \textbf{Comfort} \\
\textbf{source} & \textbf{(\%)} & \textbf{(\%)} & \textbf{(\%)} & \textbf{(\%)} & \textbf{(\%)} \\
\midrule
Learned ($\hat{a}_r$)  & 57.99 & 56.07 & 0.00 & 0.00 & 2.67 \\
Oracle ($a_r^*$)       & \textbf{32.41} & \textbf{19.79} & 0.757 & 1.55 & 12.39 \\
\bottomrule
\end{tabular}
\end{table}

\begin{figure}[!t]
\centering
\includegraphics[width=\columnwidth]{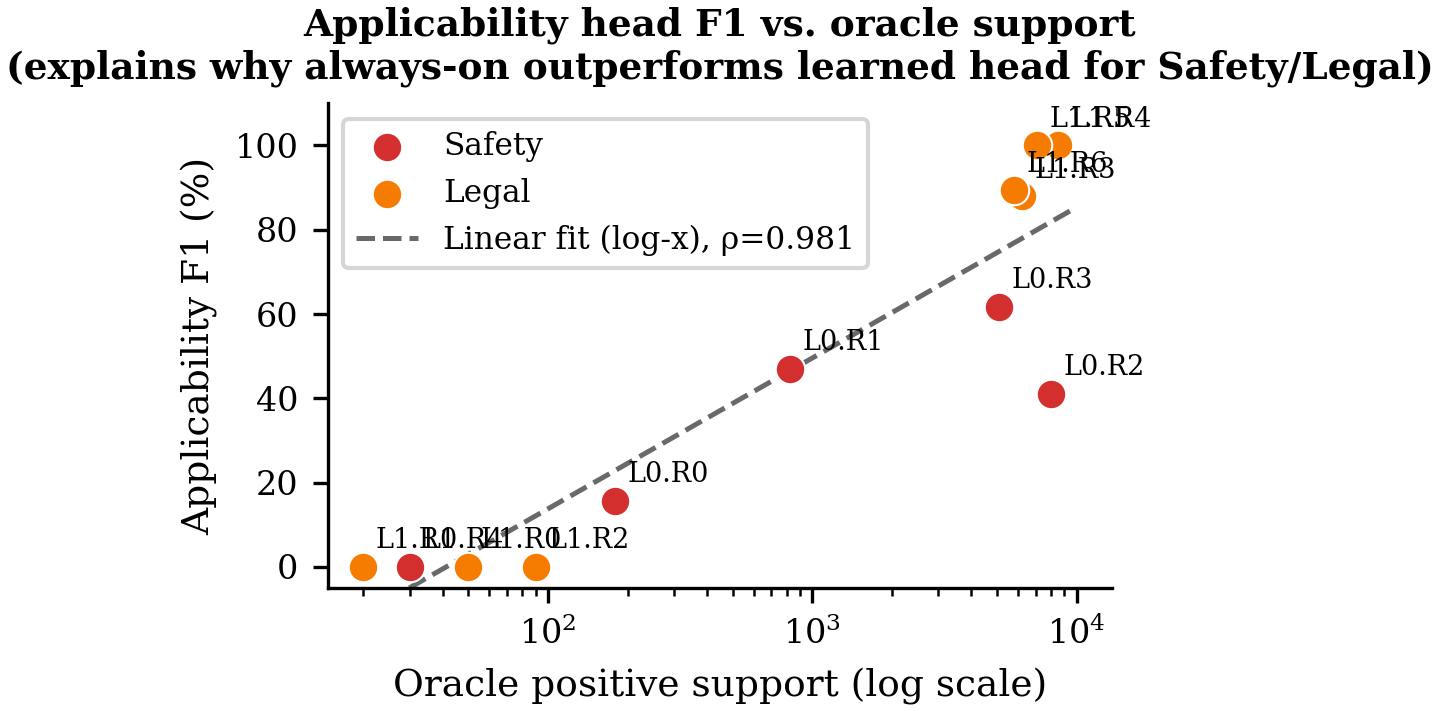}
\caption{Learned applicability-head F1 versus oracle support for all 28 rules (colored by tier). On the full evaluation set, F1 strongly correlates with support ($\rho{=}0.981$): it collapses near zero below 100 positives and typically exceeds 0.7 above 5{,}000 positives. The 10 never-invoked rules have zero support and F1$=0.0$, which motivates conservative, always-on activation for the Safety and Legal tiers.}
\vspace{-10pt}
\label{fig:f1_vs_support}
\end{figure}

\subsubsection{Applicability Learning: A Negative Result}
\label{subsubsec:applicability_negative}

Table~\ref{tab:conservative_applicability} reports the paper's clearest negative result for the learned applicability head. We compare four activation policies: the learned 3.33\,M-parameter head, a hybrid that forces Tier~0/1 (Safety, Legal) always-on (learned for Tier~2/3), a fully always-on policy, and the oracle. Each mode is used for \emph{selection}, while violations are scored under \emph{oracle} applicability (cross-evaluation), isolating the activation policy's effect on trajectory choice.

\begin{table}[t]
\centering
\caption{Conservative Applicability Comparison (Cross-Evaluation: Mode-Specific Selection, Oracle Evaluation, Lexicographic, 43{,}219 Scenarios)}
\label{tab:conservative_applicability}
\setlength{\tabcolsep}{3pt}
\begin{tabular}{lccccc}
\toprule
\textbf{Applicability mode} & \textbf{selADE} & \textbf{Safety} & \textbf{Legal} & \textbf{S+L} & \textbf{Total} \\
 & \textbf{(m)} & \textbf{(\%)} & \textbf{(\%)} & \textbf{(\%)} & \textbf{(\%)} \\
\midrule
Learned ($\hat{a}_r$) & 2.74 & 19.99 & 0.84 & 20.69 & 33.95 \\
Hybrid conservative\rlap{$^{\dagger}$} & 2.850 & 19.88 & 0.79 & 20.53 & 33.75 \\
Always-on (all rules) & 2.481 & 19.88 & 0.79 & 20.53 & 33.50 \\
Oracle ($a_r^*$) & 2.875 & 19.79 & 0.76 & 20.42 & 32.41 \\
\bottomrule
\end{tabular}
\vspace{1mm}
\caption*{\footnotesize $^{\dagger}$Hybrid conservative: always-on for Tier~0 (Safety) and Tier~1 (Legal), learned predictions for Tier~2 (Road) and Tier~3 (Comfort).  Each row uses the specified applicability mode for selection; then all rows are evaluated under Oracle applicability (full 28 rules).  Relative to the learned head, the conservative policies slightly reduce Safety and S+L violations and reduce selected-trajectory error by up to 9.3\% (2.74\,m\,$\to$\,2.481\,m); always-on also yields the lowest selADE among the conservative policies.  Oracle selection yields the lowest Total violations (32.41\%).}
\vspace{-15pt}
\end{table}

Always-on \emph{improves} selADE from 2.74\,m to 2.481\,m while slightly reducing Safety-tier violations (19.99\%\,$\to$\,19.88\%): the parameter-free policy is materially more accurate and at least as compliant as the learned head under cross-evaluation. We therefore recommend \textbf{always-on Safety/Legal as the default RECTOR configuration} and treat the learned head as a tested-but-not-recommended ablation. The negative result has two readings. Operationally, 3.33\,M parameters and Stage~1 focal-BCE training did not yield a calibrated applicability head on this task at this data scale; the under-supported rules (10 with zero positives, 14 with $<$5{,}000) drive the head toward false negatives that strictly worsen downstream selection. Methodologically, when the supervisory signal is dominated by low-prevalence rules, an always-on (or hand-specified) activation policy can dominate a learned one with no parameters and no training cost---a useful baseline for future rule-applicability research.

RECTOR can improve compliance only when the candidate set contains a relatively safe alternative. At $K{=}6$, $\sim$\textbf{3.2\%} of scenarios have no candidate with $S_0{=}0$; in those, RECTOR returns the least-violating mode and emits the \texttt{infeasible} flag from Remark~\ref{thm:infeasibility}. Downstream policies (conservative deceleration, planner replan, emergency-brake fallback) are deployment decisions, summarized in Section~\ref{Sec:Conclusion}. Under Protocol~B with always-on Safety/Legal and oracle reporting, 19.88\% of scenarios remain Safety-violating; 99.53\% of this floor is infeasibility-limited and only 0.47\% is selector misranking. Generator quality and candidate diversity therefore bound achievable compliance.

\paragraph*{Contribution of audit-only rules to the Protocol B floor.} Four rules in the Protocol~B catalog (L1.R1 priority / right-of-way, L3.R6 parking-zone violation, L3.R7 school-zone speed compliance, L3.R8 construction-zone compliance; see Table~\ref{tab:rule_catalog}) are evaluated by the oracle auditor but lack a differentiable proxy: they enter the Protocol~B total as binary auditor checks and cannot be optimized at selection time. They account for \textbf{2.6~pp} of the 32.41\% Protocol~B Total floor, primarily through L3.R8 (construction-zone compliance, $\sim$1.7~pp), which is purely a candidate-set property: no rule-aware reranking can correct it on a fixed candidate set. The remaining 29.8~pp Total floor is split between the 24-proxy-addressable rules (where the selector reduces violations by 7.91~pp vs.\ confidence-only) and the L0.R3 proxy false-negative gap (Sec.~\ref{Sec:Section_VIII_Verification_Invariants_Numerical_Stability_and_Integration_Tests}). Closing the audit-only gap requires either adding proxies for those rules or increasing generator diversity to include compliant candidates.

These gains are supported by the verification pipeline in Section~\ref{Sec:Section_VIII_Verification_Invariants_Numerical_Stability_and_Integration_Tests}: invariants pass at 99.2--100\%, post-mitigation NaN/Inf is 0.0\%, proxy--full-evaluator pairwise concordance is \textbf{0.875 on 10 variance-bearing rules} (0.948 overall including 14 trivially concordant zero-variance rules), and sensitivity is bounded under perception (${\leq}6.1$~pp) and map (${\leq}12.1$~pp) perturbations.

%-------------------------------------------------------------------------------
\subsection{Waymax Bridge: Infrastructure Verification Only}
\label{subsec:closed_loop}
%-------------------------------------------------------------------------------

All primary results above are open-loop. We verified that the RECTOR pipeline executes selected trajectories correctly through Waymax~\cite{gulino2024waymax} (\texttt{DeltaGlobal} dynamics, 10\,Hz, 55 WOMD scenarios, 80 steps each, replan every 2.0\,s, 3{,}338.5\,s end-to-end wall time). Because we use a \emph{MockLogReplayGenerator} that returns the logged trajectory as $K{=}6$ near-identical candidates ($\sigma{=}0.001$), all three selectors produce identical selections and zero overlap steps; this run is therefore an integration test of the bridge, not a selector comparison and not a reactive closed-loop study. Figure~\ref{fig:scenario_gallery} in Section~\ref{subsec:dataset} illustrates the bridge execution and is similarly non-differentiating among selectors. Reactive closed-loop evaluation with a generator that yields non-trivial candidate diversity is deliberate future work (Section~\ref{Sec:Conclusion}).

\section{Verification, Invariants, Numerical Stability, and Integration Tests}
\label{Sec:Section_VIII_Verification_Invariants_Numerical_Stability_and_Integration_Tests}

All compliance figures in this paper refer to \emph{proxy-evaluator compliance} under the Protocol~A/B stacks in Section~\ref{subsec:eval_protocols}. They are not a formal safety certificate against ego-vehicle dynamics, the deployment ODD, or any jurisdiction's traffic code. Protocol~B (full proxy catalog, oracle applicability) is the stricter in-paper reference; Protocol~A (proxy-24, predicted applicability) is the operating selection signal. Both are necessary; neither is a safety certificate.

The compliance gains in Section~\ref{Sec:Accuracy_SectionRule-Compliance_Behavior_and_Efficiency} are meaningful only if the evaluator is semantically correct, numerically stable, and reproducible. This section verifies the rule-evaluation and selection pipeline at the invariant, numerical, and integration levels, then stress-tests it under perturbations and edge cases.

We verify four invariants required by tier aggregation: non-negativity of violation scores (100\%), applicability dominance so inactive rules contribute zero (99.2\% post-mitigation, with residual boundary-adjacent map cases logged for regression), temporal consistency of per-timestep aggregation (100\%), and monotonicity of threshold-based rules (100\%). Table~\ref{tab:semantic_invariants} reports these together with end-to-end integration checks.

\begin{table*}
\centering
\caption{Semantic Invariant and Integration Test Pass Rates}
\label{tab:semantic_invariants}
\small
\begin{tabular}{lllcc}
\toprule
\textbf{Layer} & \textbf{Test / Invariant} & \textbf{Statement / Notes} & \textbf{Enforcement} & \textbf{Result} \\
\midrule
\multirow{4}{*}{Semantic}
& Non-negativity & $V_r \geq 0$ for all $r$ & ReLU/clamping & 100\% \\
& Applicability & $a_r{=}0 \Rightarrow a_r V_r{=}0$ & Multiplicative gating & 99.2\% \\
& Temporal consistency & $V_r=\mathrm{agg}_t V_r^{(t)}$ & Bounds + unit tests & 100\% \\
& Monotonicity & $\theta_1{>}\theta_2 \Rightarrow V_1{\geq} V_2$ & parameterized tests & 100\% \\
\midrule
\multirow{6}{*}{Integration}
& Tier aggregation & All tier sums match expected gated sums & End-to-end & PASS \\
& Lexicographic selection & Verified across 100 selection instances & Sampling & PASS \\
& Proxy/full agreement & On rules with both implementations & Binary comparison & 87.3\% \\
& Applicability correctness & Boundary-adjacent map edge cases & Regression & 99.2\% \\
& NaN/Inf propagation & Zero invalid values in tested runs & Monitoring & PASS \\
& Monotonicity regression & All thresholded rules preserve order & Regression suite & PASS \\
\bottomrule
\end{tabular}
\end{table*}

Distance computations, soft overlap proxies, and kinematic derivatives are protected with $\epsilon$-regularization ($\epsilon{=}10^{-6}$), degenerate-geometry checks, windowed smoothing ($W{=}3$), and finite-value bounds. After mitigation, the overall NaN/Inf rate is \textbf{0.0\%}. Tier scores consumed by the selector are well-defined; Table~\ref{tab:numerical_stability} lists pre-mitigation error modes and safeguards.

\begin{table*}
\centering
\caption{Numerical Stability Analysis}
\label{tab:numerical_stability}
\begin{tabular}{llcc}
\toprule
\textbf{Computation} & \textbf{Error Mode} & \textbf{Frequency} & \textbf{Mitigation} \\
\midrule
Point-to-segment distance & Degenerate segments & 0.3\% & Skip $<\!10^{-4}$\,m \\
Point-to-polyline queries & NaN from ill-conditioned geometry & 0.1\% & Bounds + finiteness checks \\
Soft overlap proxy & Extreme/invalid boxes & 0.05\% & Dimension checks + clamp \\
Acceleration & Velocity discontinuities & 0.8\% & Windowing + bounds \\
Jerk & Transient spikes & 1.2\% & Smoothing + thresholds \\
\midrule
\textbf{Overall NaN/Inf rate} & After mitigation & \textbf{0.0\%} & All caught + handled \\
\bottomrule
\end{tabular}
\end{table*}

Tier aggregation correctness and lexicographic determinism are verified on sampled scenarios (all pass). The full per-rule proxy--evaluator alignment table for all 24 proxied rules (1{,}000 scenarios) is in supplementary material.

A separation appears between per-rule severity correlation and candidate-ranking fidelity. Mean per-rule Spearman $\rho{=}0.25$ is low (Safety-distance proxies remain weak: L0.R0 $\rho{=}{-}0.07$, L0.R1 $\rho{=}{-}0.38$); kinematic and interaction rules align better (L3.R10 $\rho{=}0.62$, L1.R2 $\rho{=}0.77$). The operationally relevant metric is \emph{pairwise ranking accuracy}: across all $\binom{K}{2}{=}15$ pairs per scenario, mean concordance on the 10 variance-bearing rules is \textbf{0.875} (range 0.67--0.98; 0.948 averaged over all 24, including 14 zero-variance rules at trivial 1.0). Low per-rule $\rho$ coexists with high concordance because the lex selector uses summed tier scores, and aggregation cancels per-rule noise.

The collision-avoidance proxy (L0.R3) warrants explicit qualification. It reaches 0.80 pairwise concordance but flags only 19.3\% of full-evaluator-positive cases (FNR\,=\,80.7\%) at 0.0\% FPR. The high FNR is structurally consistent with the SAT-overlap formulation, which fires only on bounding-box penetration while the full evaluator counts a broader class of contact events. The gap is asymmetric (under-detection, not over-detection), so the lex selector uses L0.R3 as a \emph{sufficient condition for rejection} and never wrongly demotes a Safety-compliant candidate. Two consequences: Protocol~A/B Safety numbers under-count L0.R3 violations relative to a full-evaluator pass, and any deployment-facing reading should add a post-selection full-evaluator audit ($\sim$5\,ms/sample); always-on activation (Sec.~\ref{subsubsec:applicability_negative}) is recommended precisely because it preserves L0.R3 evaluation on every scenario. The proxy Safety numbers in Sec.~\ref{Sec:Accuracy_SectionRule-Compliance_Behavior_and_Efficiency} are not a safety certificate.

\begin{figure}[!t]
\centering
\includegraphics[width=\columnwidth]{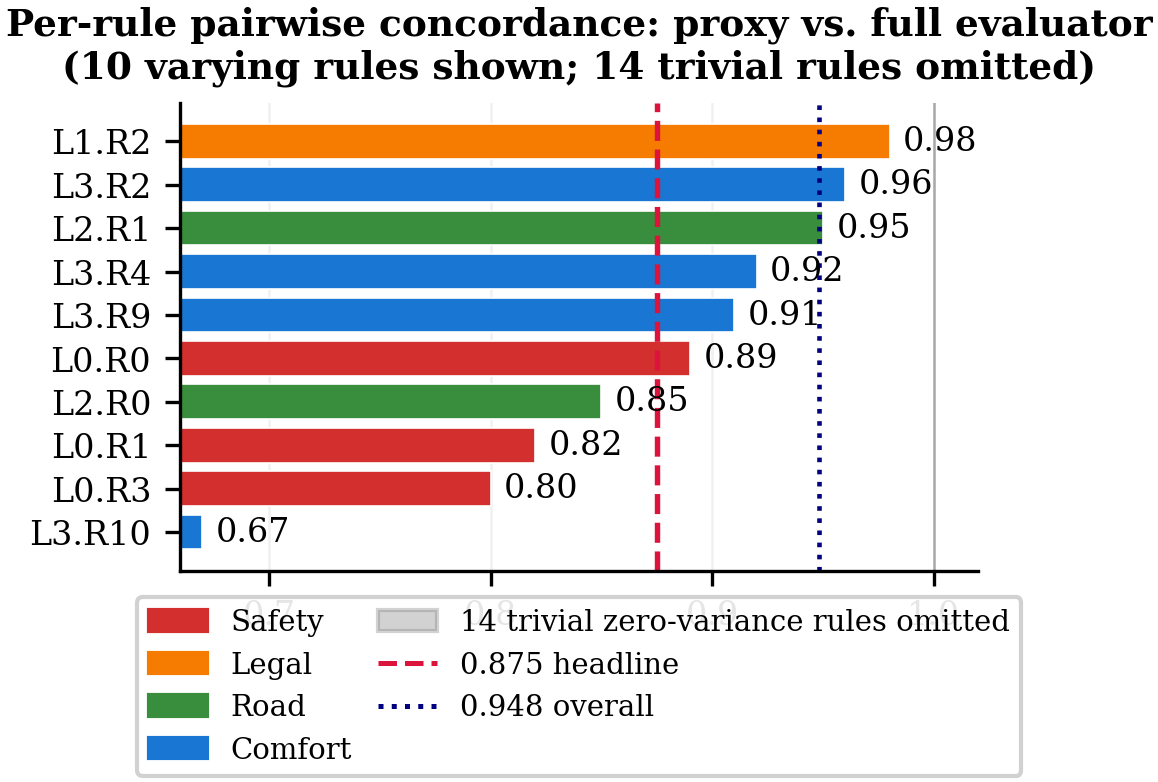}
\caption{Per-rule pairwise concordance between differentiable proxies and the full \texttt{waymo\_rule\_eval} evaluator (1{,}000 scenarios). The 10 variance-bearing rules are shown (sorted); 14 zero-variance rules are omitted as trivially concordant (1.0). Dashed lines mark mean concordance: \textbf{0.875} (variance rules) and \textbf{0.948} (all 24 rules). L0.R3 and L3.R10 are the weakest-discriminating proxies; the others exceed 0.80.}
\label{fig:concordance_bar}
\vspace{-10pt}
\end{figure}

\begin{figure}[!t]
\centering
\begin{subfigure}[t]{0.49\columnwidth}
\centering
\includegraphics[width=\textwidth]{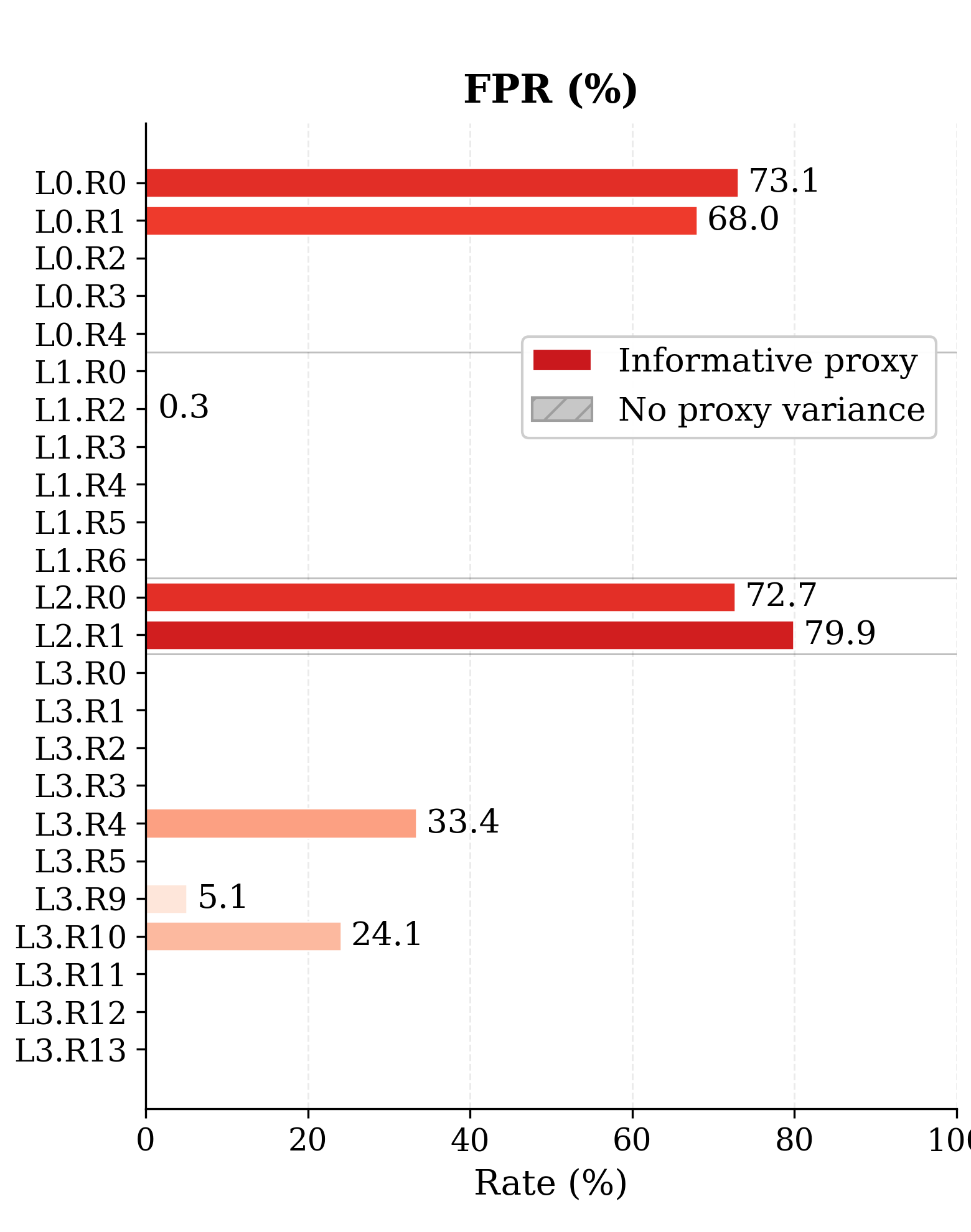}
\caption{False-positive rate (FPR)}
\end{subfigure}
\hfill
\begin{subfigure}[t]{0.49\columnwidth}
\centering
\includegraphics[width=\textwidth]{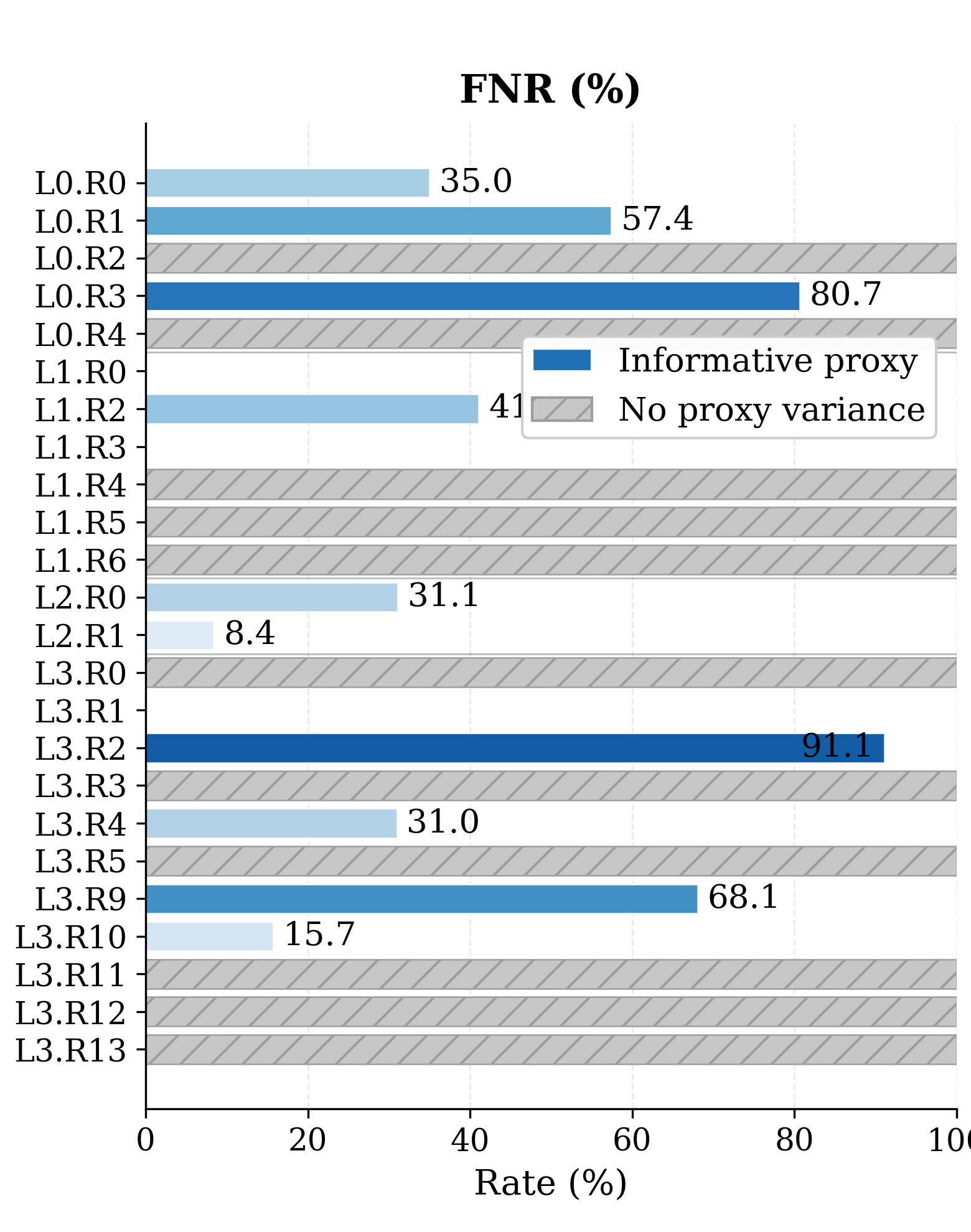}
\caption{False-negative rate (FNR)}
\end{subfigure}
\caption{Proxy reliability across all 24 proxied rules: false-positive rate (FPR) and false-negative rate (FNR), grouped by tier. Color encodes 0--100\%, and zero-variance rules are gray. L0.R3 (collision avoidance) shows 80.7\% FNR, and several Legal rules show 100\% FNR due to Protocol~A learned-applicability suppression. These high-FNR cases motivate always-on activation (Section~\ref{subsec:ablations_main}).}
\vspace{-5pt}
\label{fig:proxy_reliability}
\end{figure}

\begin{table}
\centering
\caption{Rule Evaluation Robustness to Perception and Map Perturbations (500 Scenarios)}
\label{tab:robustness_perturbations}
\setlength{\tabcolsep}{4pt} 
\begin{tabular}{c|lllcl}
\toprule
 & \makecell{\textbf{Input} / \\ \textbf{Feature}} & \makecell{\textbf{Perturbation} / \\ \textbf{Error}} & \textbf{Rule} & \textbf{Change} & \textbf{Sensitivity} \\
\midrule
\multirow{6}{*}{\rotatebox[origin=c]{90}{Perception}}
& Agent position    & 0.1--0.5\,m         & L0.R1 & $+2.1\%$ & Moderate \\
& Agent position    & 0.1--0.5\,m         & L0.R3 & $+1.8\%$ & Low \\
& Agent velocity    & 0.5--2.0\,m/s       & L0.R0  & $+3.2\%$ & Moderate \\
& Agent heading     & 2--10\textdegree    & L0.R1 & $+1.5\%$ & Low \\
& Lane centreline   & 0.1--0.3\,m         & L2.R1      & $+4.7\%$ & High \\
& Signal state      & 1--5\% flips        & L1.R0 & $+6.1\%$ & High \\
\midrule
\multirow{4}{*}{\rotatebox[origin=c]{90}{Map}}
& Lane boundaries   & Missing (20\%)      & L2.R1     & $+8.3\%$  & High \\
& Traffic signals   & Wrong (5\%)   & L1.R0 & $+12.1\%$ & High \\
& Stop signs        & Missing (10\%)      & L1.R4 & $+2.1\%$  & Low \\
& Lane centreline   & Drift (0.5\,m)      & L2.R1 & $+6.7\%$  & High \\
\bottomrule
\end{tabular}
\end{table}

To assess sensitivity to upstream errors, we inject controlled perturbations into 500 scenarios. Perception noise (position 0.1--0.5\,m; velocity 0.5--2.0\,m/s; heading 2--10\textdegree) increases violation rates by at most 6.1~pp; map errors (missing lane boundaries, wrong signal states) by up to 12.1~pp (Table~\ref{tab:robustness_perturbations}). Signal-state errors and missing lane boundaries are the highest-sensitivity inputs and the highest-priority targets for robustness hardening. A 367-case regression suite (zero-length segments, degenerate bounding boxes, track discontinuities, out-of-map queries) passes at 99.5\%; the 0.5\% residual failures are logged with diagnostic traces and do not propagate to tier scores.

% \section{Near-Term Priorities and Open Questions}
% \label{Sec:Known_Gaps_and_Next_Steps_Toward_Deployment}
% \subfile{Section_IX_Known_Gaps_and_Next_Steps_Toward_Deployment}

\section{Conclusion}
\label{Sec:Conclusion}
We presented \textsc{RECTOR}, a priority-aware re-ranking layer that selects one trajectory from a fixed multi-modal candidate set under Safety~$\succ$~Legal~$\succ$~Road~$\succ$~Comfort via $\varepsilon$-lexicographic selection over normalized proxy tier scores. From $K{=}6$ Transformer--CVAE candidates, RECTOR runs a scene-conditioned applicability mechanism (retained as a tested reference; the deployment-default is the parameter-free always-on policy for Safety/Legal), 24 differentiable proxies (pairwise concordance with the full evaluator: 0.875 on variance-bearing rules, 0.948 overall, with explicit L0.R3 qualification), and a parameter-free $\varepsilon$-lexicographic scorer. The selector's guarantee is structural priority preservation over the proxy tier scores within tolerance~$\varepsilon$; it is not a control-theoretic or jurisdictional safety certificate.

On 43{,}219 augmented WOMD \texttt{validation\_interactive} instances, rule-aware selection reduces Safety+Legal violations from \textbf{28.58\% to 20.42\%} (8.16~pp) and Total from \textbf{40.32\% to 32.41\%} (7.91~pp) versus confidence-only, with reductions on every tier. Protocol~A (24 rules, learned head) shows a larger headline Total drop, but its inflated baseline reflects Safety over-activation and Legal/Road suppression by the head; we therefore lead with Protocol~B. Weighted-sum matches lexicographic on per-tier binary compliance because cross-tier conflicts are rare in this dataset (\emph{does-not}, not \emph{cannot}); the structural lexicographic guarantee is a separate property no weight calibration can replicate. Under an adversarial protocol that injects a Safety-violating candidate with the maximum confidence, confidence-only selection picks the violator in \textbf{100\%} of scenarios; both rule-aware selectors reject it in \textbf{$\sim$96\%}---the clearest empirical separation between confidence-based and rule-aware selection. An ablation also shows that a parameter-free always-on policy Pareto-dominates the learned 3.33\,M-parameter applicability head on cross-evaluated compliance (S+L 20.53\% vs.\ 20.69\%) \emph{and} on selected-trajectory accuracy (selADE 2.481\,m vs.\ 2.737\,m); we therefore recommend always-on Safety/Legal as the default RECTOR configuration.

All compliance numbers are proxy-evaluator results and not a safety certificate; L0.R3 has 80.7\% FNR at 0.0\% FPR against the full Waymo evaluator, so a post-selection full-evaluator audit is required for any deployment-facing reading. The 32.41\% Protocol~B Total floor is dominated by candidate-set limitations: 3.2\% of scenarios contain no candidate with $S_0{=}0$, and audit-only rules (notably L3.R8) contribute $\sim$2.6~pp that no reranking on a fixed candidate set can address. Model selection was on the validation split because the WOMD test split lacks the per-rule applicability labels we require; Protocol~B selector deltas are insulated from learned-applicability bias, but absolute Protocol~A numbers carry an unquantified optimistic bias. Results are open-loop, 5\,s, U.S.\ rules, on the validation split. Four next steps follow directly: (i)~held-out test-split reporting under both protocols (contingent on label-pipeline access); (ii)~broader candidate diversity (mixture-of-experts decoders, generative ensembles) to address the 3.2\% infeasibility floor; (iii)~reactive closed-loop evaluation with a non-trivial Waymax generator (the run reported here is infrastructure verification only); (iv)~full non-differentiable post-selection audit and proxies for the four currently audit-only rules. Latency is $\sim$7.3\,ms per sample (137~samples/s), of which rule evaluation and selection add 1.4\,ms (19\%). Within these limits, RECTOR shows that explicit, priority-ordered selection is a practical and auditable mechanism for translating multimodal diversity into improved evaluator compliance.

\bibliographystyle{./IEEEtran}
\bibliography{references.bib}

\pagebreak

\appendices
%-------------------------------------------------------------------------------
\section{Architecture Specifications}
\label{app:architecture}
%-------------------------------------------------------------------------------

This appendix tabulates the full dimension specifications for each RECTOR module. All values match the released model configuration. Symbol conventions follow Table~\ref{tab:notation}.

%-------------------------------------------------------------------------------
\subsection{Scene Encoder}
\label{app:architecture:encoder}
%-------------------------------------------------------------------------------

The scene encoder maps ego history, neighbor histories, and HD map context to a shared 256-dimensional scene embedding.  Table~\ref{tab:arch_encoder} gives the full specification.

\begin{table*}[h]
\centering
\caption{Scene Encoder Architecture}
\label{tab:arch_encoder}
\small
\begin{tabular}{ll}
\toprule
\textbf{Component / Hyper-parameter} & \textbf{Value} \\
\midrule
\multicolumn{2}{l}{\textit{Input padding and masking}} \\
Max agents ($N_{\text{pad}}$) & 32 \\
Max lane segments ($L_{\text{pad}}$) & 64 \\
Points per lane segment & 20 \\
Ego history length ($T_h$) & 11 steps (1.1\,s at 10\,Hz) \\
State dimension per timestep & 4 ($x, y, \theta, v$) \\
\midrule
\multicolumn{2}{l}{\textit{Per-token encoding}} \\
Ego Conv1D kernel & 3 (causal padding) \\
Agent MLP layers & 2 ($4 \to 64 \to 256$, ReLU) \\
Lane PointNet max-pool & over 20 points per segment \\
Lane MLP layers & 2 ($2 \to 64 \to 256$, ReLU) \\
\midrule
\multicolumn{2}{l}{\textit{Transformer fusion}} \\
Model dimension ($D$) & 256 \\
Number of layers & 3 \\
Attention heads & 8 \\
FFN dimension & 1024 \\
FFN activation & ReLU \\
Dropout & 0.1 \\
Positional encoding & None (token-type embeddings) \\
\midrule
\multicolumn{2}{l}{\textit{Scene pooling}} \\
Pooling mechanism & Cross-attention (learned query $\mathbf{q}_{\text{ego}}$) \\
Output dimension & 256 \\
\midrule
Total parameters & 324,288 \\
\bottomrule
\end{tabular}
\end{table*}

%-------------------------------------------------------------------------------
\subsection{Transformer-CVAE Trajectory Decoder}
\label{app:architecture:decoder}
%-------------------------------------------------------------------------------

The decoder generates $K{=}6$ trajectory candidates over $T{=}50$ timesteps from a shared scene embedding and mode-specific goal queries.  All decoding is single-shot (non-autoregressive): the full 50-step trajectory is produced in one forward pass.  Table~\ref{tab:arch_decoder} details the configuration.

\begin{table*}[h]
\centering
\caption{Trajectory Decoder Architecture}
\label{tab:arch_decoder}
\small
\begin{tabular}{ll}
\toprule
\textbf{Component / Hyper-parameter} & \textbf{Value} \\
\midrule
\multicolumn{2}{l}{\textit{Mode configuration}} \\
Number of modes ($K$) & 6 \\
Planning horizon ($T$) & 50 steps (5.0\,s at 10\,Hz) \\
Output state & $(x, y, \theta, v)$ per timestep \\
Decoding type & Single-shot (non-autoregressive) \\
\midrule
\multicolumn{2}{l}{\textit{Goal query network (GoalNet)}} \\
Input & $\mathbf{e}_{\text{scene}} \in \mathbb{R}^{256}$ + mode query $\mathbf{q}^{(k)} \in \mathbb{R}^{256}$ \\
MLP layers & 2 ($512 \to 256$, GELU) \\
Output & goal embedding $\mathbf{g}_{\text{emb}}^{(k)} \in \mathbb{R}^{256}$ \\
\midrule
\multicolumn{2}{l}{\textit{CVAE latent variable}} \\
Latent dimension ($d_z$) & 64 \\
Prior network & MLP ($256 \to 128 \to 128$, GELU) $\to (\boldsymbol{\mu}, \log\boldsymbol{\sigma}^2)$ \\
Posterior network & MLP ($256{+}256 \to 128 \to 128$, GELU) [training only] \\
Sampling (inference) & $\mathbf{z}^{(k)} \sim \mathcal{N}(\boldsymbol{\mu}_{\text{prior}}, \text{diag}(\boldsymbol{\sigma}^2_{\text{prior}}))$, independent per mode \\
\midrule
\multicolumn{2}{l}{\textit{Transformer decoder}} \\
Input dimension & $256 + 256 + 64 = 576$ (scene + goal + latent) \\
Number of layers & 4 \\
Attention heads & 8 \\
FFN dimension & 1024 \\
FFN activation & GELU \\
Dropout & 0.1 \\
\midrule
\multicolumn{2}{l}{\textit{Output heads}} \\
Regression head & 2 layers ($256 \to 256 \to 4T$, GELU + linear) \\
Trajectory scale & 50.0 (output multiplier) \\
Confidence head & Linear ($256 \to K$) followed by softmax \\
\midrule
Total parameters & 5,182,165 \\
\bottomrule
\end{tabular}
\end{table*}

%-------------------------------------------------------------------------------
\subsection{Applicability Head}
\label{app:architecture:applicability}
%-------------------------------------------------------------------------------

The applicability head uses a TierAwareBlock architecture: each of the $R{=}28$ rules has an independent query vector that attends over the scene embedding.  Table~\ref{tab:arch_applicability} gives the full specification.

\begin{table*}[h]
\centering
\caption{Applicability Head Architecture}
\label{tab:arch_applicability}
\small
\begin{tabular}{ll}
\toprule
\textbf{Component / Hyper-parameter} & \textbf{Value} \\
\midrule
\multicolumn{2}{l}{\textit{TierAwareBlock (one per rule)}} \\
Rule query dimension & 256 (learned, one per rule) \\
Cross-attention heads & 4 \\
Cross-attention input & $\mathbf{q}_r \in \mathbb{R}^{256}$, keys/values from $\mathbf{e}_{\text{scene}}$ \\
\midrule
\multicolumn{2}{l}{\textit{Per-rule MLP}} \\
Layer 1 & $256 \to 384$ (GELU) \\
Layer 2 & $384 \to 1$ (linear) \\
Output activation & Sigmoid \\
\midrule
\multicolumn{2}{l}{\textit{Binarization thresholds}} \\
Safety (Tier~0) & 0.05 \\
Legal (Tier~1) & 0.15 \\
Road (Tier~2) & 0.30 \\
Comfort (Tier~3) & 0.50 \\
\midrule
\multicolumn{2}{l}{\textit{Bias initialization (per tier)}} \\
Safety (Tier~0) & $+2.0$ (strong prior toward active detection) \\
Legal (Tier~1) & $+1.0$ (moderate prior toward active) \\
Road (Tier~2) & $0.0$ (neutral) \\
Comfort (Tier~3) & $-1.0$ (prior toward inactive) \\
\midrule
\multicolumn{2}{l}{\textit{Training}} \\
Loss & Focal binary cross-entropy, $\gamma{=}2.0$ \\
FN penalty multipliers & Safety $8\times$, Legal $4\times$, Road $2\times$, Comfort $1\times$ \\
\midrule
Total parameters (head + scorer) & 3,327,289 \\
\bottomrule
\end{tabular}
\vspace{1mm}
\caption*{\footnotesize The tiered scorer (24 proxy evaluations + lexicographic aggregation) has zero trainable parameters; the 3.33\,M count is entirely the applicability head. The recommended deployment policy uses always-on activation ($\hat{a}_r{=}1$ for all $r$) for Safety and Legal tiers, bypassing the head for those rules; the head is used only for Road and Comfort tiers where false-positive activation incurs an accuracy cost (see Section~\ref{Sec:Accuracy_SectionRule-Compliance_Behavior_and_Efficiency}).}
\end{table*}

%-------------------------------------------------------------------------------
\section{Differentiable Proxy Symbolic Forms}
\label{app:proxies}
%-------------------------------------------------------------------------------

This appendix provides the exact functional form, kinematic variables, softness parameter $\kappa_r$, and activation condition for each of the 24~differentiable proxy rules used in RECTOR.  All proxies produce a raw severity $V_r(\tau;s) \ge 0$ that is normalized via the exponential cost map (Eq.~\ref{eq:exponential_cost}); the per-rule $\kappa_r$ and violation threshold appear in Table~\ref{tab:proxy_constants}.  Proxies are evaluated at 10\,Hz; higher-order kinematics (acceleration, jerk) are computed by finite differences with windowed smoothing ($W{=}3$).  Rules are grouped by tier to match the priority structure.

%-------------------------------------------------------------------------------
\subsection*{Tier 0 — Safety}
%-------------------------------------------------------------------------------

\noindent\textbf{L0.R0 — Safe longitudinal distance.}
Let $d_t^{\text{long}}$ be the front-bumper-to-front-bumper distance to the nearest lead vehicle in the same lane at timestep~$t$, and $v_t$ the ego speed.  The required headway is $d_t^{\text{req}} = v_t \cdot \Delta_{\text{gap}}$ with $\Delta_{\text{gap}}{=}2.0$\,s.  The proxy severity is:
\begin{equation}
V_{\text{L0.R0}} = \sum_{t=1}^{T} \max\!\bigl(0,\, d_t^{\text{req}} - d_t^{\text{long}}\bigr).
\end{equation}
\textit{Activation:} lead vehicle detected, $v_t \ge 0.3$\,m/s.

\medskip
\noindent\textbf{L0.R1 — Safe lateral clearance.}
Let $c_t^{j}$ be the edge-to-edge lateral distance between the ego bounding box and agent~$j$ at timestep~$t$, and $c_{\min}^{j}$ the type-dependent minimum clearance (vehicle: 0.5\,m; cyclist: 1.0\,m; pedestrian: 1.5\,m).  The proxy severity is:
\begin{equation}
V_{\text{L0.R1}} = \sum_{t=1}^{T} \max_{j \in \mathcal{A}_t} \max\!\bigl(0,\, c_{\min}^{j} - c_t^{j}\bigr).
\end{equation}
\textit{Activation:} agents within 50\,m lateral range.

\medskip
\noindent\textbf{L0.R2 — Crosswalk occupancy.}
Let $A_t$ be the overlap area (m$^2$) between the ego OBB and a crosswalk polygon at timestep~$t$, computed via SAT.  The proxy severity is:
\begin{equation}
V_{\text{L0.R2}} = \sum_{t=1}^{T} A_t.
\end{equation}
\textit{Activation:} crosswalk polygon in map, pedestrian moving $\ge 0.3$\,m/s within 5\,m buffer.

\medskip
\noindent\textbf{L0.R3 — Collision avoidance (overlap).}
Let $p_t^{j,\text{long}}$ and $p_t^{j,\text{lat}}$ be the SAT penetration depths along the longitudinal and lateral axes between the ego OBB and agent~$j$ at timestep~$t$.  The proxy severity is:
\begin{equation}
V_{\text{L0.R3}} = \sum_{t=1}^{T}\sum_{j \in \mathcal{A}_t} \min\!\bigl(p_t^{j,\text{long}},\, p_t^{j,\text{lat}}\bigr).
\end{equation}
\textit{Activation:} other agents within 50\,m spatial pre-filter; penetration threshold $0.01$\,m.

\medskip
\noindent\textbf{L0.R4 — VRU clearance.}
Let $d_t^{j}$ be the edge-to-edge distance between the ego OBB and VRU~$j$ (pedestrian/cyclist) at timestep~$t$, and $r^{j}$ the type-dependent minimum clearance (pedestrian: 2.0\,m; cyclist: 1.5\,m).  The proxy severity is:
\begin{equation}
V_{\text{L0.R4}} = \sum_{t=1}^{T}\max_{j \in \mathcal{V}_t} \max\!\bigl(0,\, r^{j} - d_t^{j}\bigr).
\end{equation}
\textit{Activation:} VRUs present, ego speed $\ge 1.0$\,m/s.

%-------------------------------------------------------------------------------
\subsection*{Tier 1 — Legal}
%-------------------------------------------------------------------------------

\noindent\textbf{L1.R0 — Traffic signal compliance.}
Let $r_t \in \{0,1\}$ indicate a red signal phase at timestep~$t$, $y_t$ a yellow phase, $d_t^{\text{stop}}$ the distance to the stop line, and $a_t$ the ego acceleration.  The proxy severity accumulates a speed-weighted red-light penalty and a scaled yellow-acceleration penalty:
\begin{align}
V_{\text{L1.R0}} = \sum_{t=1}^{T}\Bigl[r_t &\cdot \mathbb{1}[d_t^{\text{stop}} < 5\,\text{m}] \cdot \min\!\bigl(1, v_t/10\bigr)\\
 &+ 0.3\, y_t \cdot \mathbb{1}[d_t^{\text{stop}} < 30\,\text{m}] \cdot \min\!\bigl(1, a_t/2\bigr)\Bigr].\nonumber
\end{align}
\textit{Activation:} stop-line geometry and per-timestep signal state available.

\medskip
\noindent\textbf{L1.R2 — Speed limit adherence.}
Let $v_t$ be ego speed and $v_t^{\lim}$ the local speed limit with tolerance $\Delta v{=}1.0$\,m/s.  The proxy severity is:
\begin{equation}
V_{\text{L1.R2}} = \sum_{t=1}^{T} \max\!\bigl(0,\, v_t - v_t^{\lim} - \Delta v\bigr).
\end{equation}
\textit{Activation:} speed-limit data available (map or default).

\medskip
\noindent\textbf{L1.R3 — Red-light stop compliance.}
Let $\text{cross}_t$ be a differentiable stop-line crossing measure and $r_t$ a soft red-signal indicator.  The proxy severity counts red-phase stop-line crossings:
\begin{equation}
V_{\text{L1.R3}} = \sum_{t=1}^{T} r_t \cdot \sigma_\alpha\!\bigl(\text{cross}_t - \text{cross}_{t-1}\bigr),
\end{equation}
where $\sigma_\alpha$ is a soft step with temperature $\alpha$. \textit{Activation:} red signal phases exist.

\medskip
\noindent\textbf{L1.R4 — Stop-sign compliance.}
Let $v_t^{\text{stop}}$ be the ego speed within a 5\,m window of the stop line and $d_t^{\text{past}}$ the depth past the line.  The proxy severity is:
\begin{equation}
V_{\text{L1.R4}} = v_t^{\text{stop,max}} \cdot \bigl(1 + d_t^{\text{past,max}} / 5\bigr).
\end{equation}
\textit{Activation:} stop-sign annotation within approach radius.

\medskip
\noindent\textbf{L1.R5 — Crosswalk yield.}
Let $\text{TTC}_t^j$ be time-to-collision with pedestrian~$j$ in a crosswalk, $v_t$ ego speed, and $\rho_t$ proximity to the crosswalk.  The proxy severity is:
\begin{align}
V_{\text{L1.R5}} = \min\!\bigl(5, (3 - &\text{TTC}_{\min})\cdot 2\bigr)+ \min\!\bigl(3, v_t/10\bigr)\\
&+ \min\!\bigl(2, (\rho_{\min} - d_{\text{cw}})/7.5\bigr),\nonumber
\end{align}
where $\text{TTC}_{\min}$ is the minimum TTC (safe threshold 3\,s), $d_{\text{cw}} {=}15$\,m crosswalk proximity. \textit{Activation:} crosswalk in map, VRU present, vehicle moving.

\medskip
\noindent\textbf{L1.R6 — Wrong-way driving.}
Let $\phi_t$ be the heading mismatch between the ego heading and the lane centerline direction, $v_t$ ego speed, and $\Delta t_{\text{viol}}$ the continuous duration of heading mismatch.  The proxy severity is:
\begin{align}
V_{\text{L1.R6}} = \phi_{\max,t} / 90^\circ \cdot 0.4 &+ \min(1, \Delta t_{\text{viol}} / 2) \cdot 0.4 \\
&+ \min(1, v_t / 10) \cdot 0.2.\nonumber
\end{align}
\textit{Activation:} lane centerline available, ego speed $\ge 0.5$\,m/s.

%-------------------------------------------------------------------------------
\subsection*{Tier 2 — Road}
%-------------------------------------------------------------------------------

\noindent\textbf{L2.R0 — Drivable surface.}
Let $\delta_t$ be the signed distance from the nearest drivable-surface boundary (negative = outside).  The proxy severity is:
\begin{equation}
V_{\text{L2.R0}} = \sum_{t=1}^{T} \max\!\bigl(0,\, -\delta_t - 0.5\bigr),
\end{equation}
with a 0.5\,m buffer inside the boundary. \textit{Activation:} lane centerline data available, vehicle moving.

\medskip
\noindent\textbf{L2.R1 — Lane departure.}
Let $d_t^{\text{lat}}$ be the signed lateral offset from the lane centerline and $w$ the half-lane width (1.75\,m).  The proxy severity is:
\begin{equation}
V_{\text{L2.R1}} = \sum_{t=1}^{T} \max\!\bigl(0,\, |d_t^{\text{lat}}| - w - m\bigr),
\end{equation}
where $m{=}0.05$\,m is a soft margin. \textit{Activation:} lane centerline and boundary type available.

%-------------------------------------------------------------------------------
\subsection*{Tier 3 — Comfort}
%-------------------------------------------------------------------------------

\noindent\textbf{L3.R0 — Smooth longitudinal acceleration.}
Let $a_t^{\text{long}}$ be the longitudinal acceleration at timestep~$t$ and $j_t$ the longitudinal jerk ($\Delta a / \Delta t$).  The proxy severity is:
\begin{equation}
V_{\text{L3.R0}} = \sum_{t=1}^{T} \max\!\bigl(0, |a_t^{\text{long}}| - 2.0\bigr) + \max\!\bigl(0, |j_t| - 2.0\bigr).
\end{equation}
\textit{Activation:} ego speed $\ge 0.5$\,m/s, $\ge 3$ frames of data.

\medskip
\noindent\textbf{L3.R1 — Smooth braking deceleration.}
Let $a_t^{\text{dec}} = \max(0, -a_t^{\text{long}})$ be the deceleration magnitude.  The proxy severity accumulates excess deceleration relative to a comfort limit $a_{\text{comfort}}{=}1.5$\,m/s$^2$:
\begin{equation}
V_{\text{L3.R1}} = \sum_{t=1}^{T} \max\!\bigl(0,\, a_t^{\text{dec}} - a_{\text{comfort}}\bigr) \cdot \Delta t.
\end{equation}
\textit{Activation:} ego speed $\ge 1.0$\,m/s.

\medskip
\noindent\textbf{L3.R2 — Smooth lateral steering.}
Let $\dot\psi_t$ be the heading rate (rad/s) and $\ddot\psi_t$ the angular jerk.  The proxy severity is:
\begin{align}
V_{\text{L3.R2}} = \sum_{t=1}^{T} &\max\!\bigl(0, |\dot\psi_t|(180/\pi) - 15\bigr) \\
&+ \max\!\bigl(0, |\ddot\psi_t|(180/\pi) - 15\bigr).\nonumber
\end{align}
\textit{Activation:} turning ($|\dot\psi_t| > 0.01$\,rad/s), vehicle moving.

\medskip
\noindent\textbf{L3.R3 — Speed consistency.}
Let $v_t$ be ego speed and $\sigma^2_{t,W}$ its variance over a rolling 2\,s window of size $W$.  The proxy severity penalises excess variance and oscillation:
\begin{align}
V_{\text{L3.R3}} = \sum_{t=W}^{T} \max\!\bigl(0, &\sigma_{t,W} - 2.0\bigr)\\
&+ \lambda_{\text{osc}} \cdot \mathbb{1}[\text{sign-changes} > 6],\nonumber
\end{align}
where $\lambda_{\text{osc}}{=}1.0$. \textit{Activation:} ego speed $\ge 0.5$\,m/s, $\ge W$ frames.

\medskip
\noindent\textbf{L3.R4 — Lane-change smoothness (jerk).}
Let $a_t^{\text{lat}}$ be lateral acceleration.  The proxy severity is:
\begin{equation}
V_{\text{L3.R4}} = \sum_{t=1}^{T} \max\!\bigl(0,\, |a_t^{\text{lat}}| - 1.5\bigr).
\end{equation}
\textit{Activation:} lateral velocity $\ge 0.1$\,m/s.

\medskip
\noindent\textbf{L3.R5 — Left-turn gap acceptance.}
Let $\text{TTC}_t^j$ be the time-to-collision with oncoming vehicle~$j$ during a left turn at timestep~$t$.  The proxy severity is:
\begin{equation}
V_{\text{L3.R5}} = \sum_{t \in \mathcal{T}_{\text{turn}}} \sum_{j \in \mathcal{A}_t^{\text{oncoming}}} \max\!\bigl(0,\, 1/\text{TTC}_t^j - 1/4\bigr),
\end{equation}
where the safe TTC threshold is 4\,s. \textit{Activation:} left turn (heading change $\ge 15^\circ$), oncoming vehicles within 50\,m.

\medskip
\noindent\textbf{L3.R9 — Cooperative lane change.}
Let $g_t^j$ be the time gap to vehicle~$j$ and $a_t^{\text{lat}}$ lateral acceleration during a lane change.  The proxy severity is:
\begin{equation}
V_{\text{L3.R9}} = 0.5 \sum_{t} \mathbb{1}[g_t^j < 2.0] + 0.2 \sum_{t} \mathbb{1}[|a_t^{\text{lat}}| > 0.5].
\end{equation}
\textit{Activation:} lateral displacement $\ge 2.5$\,m (lane change detected).

\medskip
\noindent\textbf{L3.R10 — Following distance.}
Let $g_t^{\text{time}}$ be the time gap to the nearest lead vehicle.  The proxy severity is:
\begin{equation}
V_{\text{L3.R10}} = \sum_{t=1}^{T} \max\!\bigl(0,\, 1 - g_t^{\text{time}}/2.0\bigr),
\end{equation}
where 2.0\,s is the two-second rule threshold. \textit{Activation:} lead vehicle in scene, ego speed $\ge 0.3$\,m/s.

\medskip
\noindent\textbf{L3.R11 — Intersection negotiation.}
Let $g_t^j$ be the entry gap to intersecting vehicle~$j$ and $v_t$ ego speed through the intersection (radius $r{=}20$\,m).  The proxy severity is:
\begin{equation}
V_{\text{L3.R11}} = 0.5\sum_t \mathbb{1}[g_t^j < 3.0] + 0.2\sum_t \mathbb{1}[v_t > 8.0].
\end{equation}
\textit{Activation:} $\ge 2$ agents, sustained intersection presence.

\medskip
\noindent\textbf{L3.R12 — Pedestrian interaction.}
Let $d_t^j$ be the ego-to-pedestrian edge distance and $v_t$ ego speed.  The proxy severity is:
\begin{equation}
V_{\text{L3.R12}} = \sum_{j \in \mathcal{V}_t^{\text{ped}}} \max\!\bigl(0, 1.5 - d_t^j\bigr) + \mathbb{1}[d_t^j < 3.0] \cdot \max\!\bigl(0, v_t - 6.7\bigr).
\end{equation}
\textit{Activation:} pedestrians within 10\,m.

\medskip
\noindent\textbf{L3.R13 — Cyclist interaction.}
Identical structure to L3.R12 with type-specific clearance (1.5\,m) and speed limit near cyclists (8.9\,m/s):
\begin{align}
V_{\text{L3.R13}} = \sum_{j \in \mathcal{V}_t^{\text{cyc}}} &\max\!\bigl(0, 1.5 - d_t^j\bigr) \\
&+ \mathbb{1}[d_t^j < 3.0] \cdot \max\!\bigl(0, v_t - 8.9\bigr).\nonumber
\end{align}
\textit{Activation:} cyclists within 10\,m.

%-------------------------------------------------------------------------------
\section{Experimental Subset Sampling Procedures}
\label{app:sampling}
%-------------------------------------------------------------------------------

This appendix specifies the exact sampling procedure and random seed for every evaluation subset used in the paper.  All subsets are drawn from WOMD v1.3.0 \texttt{validation\_interactive} unless stated otherwise.  The canonical evaluation uses the augmented validation split; all subset sizes below refer to base scenarios (before per-rule augmentation).

%-------------------------------------------------------------------------------
\subsection{Canonical Validation Benchmark (12{,}800 Scenarios)}
\label{app:sampling:canonical}
%-------------------------------------------------------------------------------

The primary selector benchmark evaluates on the full augmented \texttt{validation\_interactive} split, producing 43{,}219 evaluation instances.

\begin{itemize}[leftmargin=2em]
\item \textbf{Source:} All 150 WOMD v1.3.0 validation shards (\texttt{validation\_interactive/*.tfrecord}).
\item \textbf{Base scenarios:} 12{,}800 (all interactive scenarios, no downsampling).
\item \textbf{Augmented instances:} 43{,}219 (augmentation factor $\approx 3.4\times$; see Table~\ref{tab:data_augmentation}).
\item \textbf{Seed:} Not applicable (all scenarios used).
\end{itemize}

%-------------------------------------------------------------------------------
\subsection{Proxy--Full Evaluator Alignment Subset (1{,}000 Scenarios)}
\label{app:sampling:proxy_alignment}
%-------------------------------------------------------------------------------

Used in Table~\ref{tab:proxy_correlation_detailed} and Section~\ref{Sec:Accuracy_SectionRule-Compliance_Behavior_and_Efficiency}.

\begin{itemize}[leftmargin=2em]
\item \textbf{Source:} \texttt{validation\_interactive} base scenarios.
\item \textbf{Sampling:} Stratified random sample with \texttt{random.seed(42)}; strata defined by interaction type (v2v, v2p, v2c, other) in proportion to their occurrence in the full split.
\item \textbf{Coordinate alignment:} World-frame trajectories; no ego-centric re-centering applied, to match the coordinate frame of the full evaluator.
\item \textbf{Evaluation:} All 24 proxy rules + full-evaluator counterparts for all $\binom{6}{2}{=}15$ candidate pairs per scenario.
\end{itemize}

%-------------------------------------------------------------------------------
\subsection{Perturbation Robustness Subset (500 Scenarios)}
\label{app:sampling:perturbation}
%-------------------------------------------------------------------------------

Used in Table~\ref{tab:robustness_perturbations} and Section~\ref{Sec:Section_VIII_Verification_Invariants_Numerical_Stability_and_Integration_Tests}.

\begin{itemize}[leftmargin=2em]
\item \textbf{Source:} \texttt{validation\_interactive} base scenarios.
\item \textbf{Sampling:} Random sample with \texttt{numpy.random.seed(42)}, selecting the first 500 scenarios after shuffle.
\item \textbf{Perturbation protocol:} Each scenario evaluated 10 times per perturbation type (position: 3 noise levels; velocity: 3 noise levels; heading: 3 noise levels; map: 4 error types). Results averaged across repetitions within each perturbation type.
\item \textbf{Reported metric:} Change in per-rule violation rate (percentage points) relative to unperturbed evaluation.
\end{itemize}

%-------------------------------------------------------------------------------
\subsection{Protocol~B Audit Subset (337 Scenarios; Diagnostic Only)}
\label{app:sampling:protocolB}
%-------------------------------------------------------------------------------

This subset was used for early spot-checking and diagnostic auditing of Protocol~B (full 28 rules, oracle applicability). The main-text Protocol~B tables and ablations now report the full 43{,}219-instance augmented validation benchmark defined in Appendix~\ref{app:sampling:canonical}.

\begin{itemize}[leftmargin=2em]
\item \textbf{Source:} \texttt{validation\_interactive} base scenarios.
\item \textbf{Sampling:} Seeded random sample (\texttt{seed=42}) ensuring representation of all rule-triggering scenario types identified during dataset analysis.
\item \textbf{Rationale for size:} 337 scenarios provides $>97\%$ statistical power to detect a 5\,pp compliance difference at $\alpha{=}0.05$ (two-sided McNemar, $\pi_0{=}0.20$).
\item \textbf{Oracle labels:} Per-rule applicability ground truth from \texttt{waymo\_rule\_eval} metadata.
\end{itemize}

%-------------------------------------------------------------------------------
\subsection{Applicability Head Diagnostic Sample (500 GT / 337 RECTOR Scenarios)}
\label{app:sampling:applicability}
%-------------------------------------------------------------------------------

Used in Section~\ref{subsec:rule_catalog} and Figure~\ref{fig:f1_vs_support} (diagnostic analysis only). \emph{Note:} This sample was historically used to compute applicability F1 scores, but Table~\ref{tab:rule_catalog} now reports applicability status (18 invoked, 10 never invoked) based on the full 43,219-instance evaluation set.

\begin{itemize}[leftmargin=2em]
\item \textbf{GT sample:} 500 randomly sampled \texttt{validation\_interactive} scenarios (\texttt{seed=42}) with oracle applicability labels from \texttt{waymo\_rule\_eval}.
\item \textbf{RECTOR sample:} The 337-scenario Protocol~B subset, for which RECTOR predicted-applicability labels and oracle labels are both available.
\item \textbf{F1 computation (diagnostic):} Macro F1 per rule; binary positive = rule applies.  Rules with zero positive support in the sample report F1 = 0.0\%.  This diagnostic F1 is superseded by the full-set support counts in Table~\ref{tab:rule_catalog} and Figure~\ref{fig:f1_vs_support}.
\end{itemize}

%-------------------------------------------------------------------------------
\subsection{Edge-Case Regression Suite (367 Cases)}
\label{app:sampling:edgecases}
%-------------------------------------------------------------------------------

Used in Section~\ref{Sec:Section_VIII_Verification_Invariants_Numerical_Stability_and_Integration_Tests} (99.5\% pass rate).

\begin{itemize}[leftmargin=2em]
\item \textbf{Composition:} Manually constructed edge cases covering: zero-length segments (42), degenerate bounding boxes (31), track discontinuities (78), out-of-map queries (56), stationary ego (38), missing map prerequisites (67), high-acceleration profiles (55).
\item \textbf{Pass criterion:} Rule evaluation returns a finite, non-negative severity without triggering NaN/Inf propagation or assertion failures.
\item \textbf{Residual failures (0.5\%):} Logged with diagnostic traces; do not propagate to tier scores.
\end{itemize}

%-------------------------------------------------------------------------------
\subsection{Lexicographic Selection Verification (100 Instances)}
\label{app:sampling:lex_verification}
%-------------------------------------------------------------------------------

Used in Table~\ref{tab:semantic_invariants} (``Lexicographic selection: Verified across 100 selection instances'').

\begin{itemize}[leftmargin=2em]
\item \textbf{Sampling:} 100 randomly drawn scenarios (\texttt{seed=0}) from the 12{,}800-scenario validation split, with all $K{=}6$ candidates and their tier-score vectors.
\item \textbf{Verification:} The selector output is compared exhaustively against all possible candidate orderings; enumeration-order invariance is checked by re-running with permuted candidate indices.
\end{itemize}

%-------------------------------------------------------------------------------
\section{Statistical Methods}
\label{app:statistics}
%-------------------------------------------------------------------------------

This appendix documents the hypothesis tests, confidence interval procedures, and software references used in Sections~\ref{Sec:Experiments} and~\ref{Sec:Accuracy_SectionRule-Compliance_Behavior_and_Efficiency}.

%-------------------------------------------------------------------------------
\subsection{McNemar's Test (Paired Binary Compliance)}
\label{app:statistics:mcnemar}
%-------------------------------------------------------------------------------

McNemar's test is applied to compare whether two selectors (e.g., lexicographic vs.\ confidence-only) differ significantly in binary rule compliance on the same set of scenarios.

\noindent\textbf{Setup.}  For each scenario $i$, let $b_i = \mathbb{1}[\text{selector A violates}]$ and $c_i = \mathbb{1}[\text{selector B violates}]$. The discordant counts are $B = \sum_i b_i(1-c_i)$ (A violates, B complies) and $C = \sum_i c_i(1-b_i)$ (B violates, A complies).

\noindent\textbf{Statistic.}  For $B + C \ge 25$, the continuity-corrected statistic is:
\begin{equation}
\chi^2 = \frac{(|B - C| - 1)^2}{B + C},
\end{equation}
which is compared against $\chi^2_1$.  For small $B + C < 25$, the exact binomial test is used ($p$-value = $2 \cdot \text{Binomial}(B; B+C, 0.5)$ two-sided).

\noindent\textbf{Significance threshold.}  $\alpha = 0.05$ (two-sided).  All reported $p$-values are uncorrected; with 28 rules the Bonferroni threshold is $\alpha' = 0.0018$.

%H% \noindent\textbf{Software.}  \texttt{scipy.stats.mcnemar} (SciPy v1.11.4, Python 3.10).

%-------------------------------------------------------------------------------
\subsection{Wilcoxon Signed-Rank Test (Paired Continuous Metrics)}
\label{app:statistics:wilcoxon}
%-------------------------------------------------------------------------------

The Wilcoxon signed-rank test is applied to compare per-scenario selADE or tier scores between two selectors.  It is preferred over the paired $t$-test because selADE distributions are heavy-tailed (Section~\ref{Sec:Problem_Formulation}).

\noindent\textbf{Statistic.}  Let $d_i = x_i^A - x_i^B$ be per-scenario differences.  Ties ($d_i = 0$) are excluded; ranks are assigned to $|d_i|$ and the signed-rank sum $W^+$ is computed.

\noindent\textbf{Significance threshold.}  $\alpha = 0.05$ (two-sided).

%H% \noindent\textbf{Software.}  \texttt{scipy.stats.wilcoxon} (SciPy v1.11.4).

%-------------------------------------------------------------------------------
\subsection{Kolmogorov--Smirnov Test (Distribution Comparison)}
\label{app:statistics:ks}
%-------------------------------------------------------------------------------

The two-sample KS test is used to compare the distribution of per-scenario compliance scores between evaluation subsets (e.g., validation vs.\ test split, or protocol-A vs.\ protocol-B).

\noindent\textbf{Statistic.}  $D = \sup_x |F_1(x) - F_2(x)|$, where $F_1, F_2$ are empirical CDFs.

%H% \noindent\textbf{Software.}  \texttt{scipy.stats.ks\_2samp} (SciPy v1.11.4).

%-------------------------------------------------------------------------------
\subsection{Bootstrap Confidence Intervals}
\label{app:statistics:bootstrap}
%-------------------------------------------------------------------------------

Bootstrap 95\% confidence intervals are reported for violation rates and selADE.

\noindent\textbf{Procedure.}
\begin{enumerate}[leftmargin=2em]
\item Draw $B{=}10{,}000$ bootstrap resamples (with replacement) from the $N$ per-scenario measurements.
\item Compute the statistic (e.g., violation rate) on each resample.
\item Report the 2.5th and 97.5th percentiles of the bootstrap distribution as the 95\% CI.
\end{enumerate}

%H% \noindent\textbf{Software.}  \texttt{scipy.stats.bootstrap} (SciPy v1.11.4, method=\texttt{percentile}, \texttt{n\_resamples=10000}, \texttt{random\_state=42}).

%-------------------------------------------------------------------------------
\subsection{Spearman Rank Correlation}
\label{app:statistics:spearman}
%-------------------------------------------------------------------------------

Spearman $\rho$ is used to measure rank correlation between proxy severity and full-evaluator severity across all $K{=}6$ candidates per scenario (Table~\ref{tab:proxy_correlation_detailed}).

%H% \noindent\textbf{Software.}  \texttt{scipy.stats.spearmanr} (SciPy v1.11.4).

\noindent\textbf{Note on zero-variance rules.}  When one or both evaluators return zero variance (a constant value across all candidates), $\rho$ is undefined and reported as ``---''.  These rules achieve pairwise concordance $= 1.0$ trivially and are excluded from the mean $\rho$ computation; the mean $\rho{=}0.25$ is computed over the 10 rules with measurable variance in both evaluators.

%----------------------------------------
% % \subfile{Appendix_E_PluginContract}

\end{document}